\documentclass[10pt,twocolumn,letterpaper]{article}

\usepackage[pagenumbers]{cvpr} 

\usepackage[dvipsnames]{xcolor}
\usepackage{bm}
\usepackage{booktabs}
\usepackage{amsthm}
\usepackage{multirow}

\usepackage{amsmath}

\newtheorem{theorem}{Theorem}

\definecolor{cvprblue}{rgb}{0.21,0.49,0.74}
\usepackage[pagebackref,breaklinks,colorlinks,citecolor=cvprblue]{hyperref}
\usepackage[accsupp]{axessibility} 
\usepackage{xcolor}

\def\paperID{9730} 
\def\confName{CVPR}
\def\confYear{2024}

\title{Degrees of Freedom Matter: \\ Inferring Dynamics from Point Trajectories}

\newcommand{\myparagraph}[1]{\vspace{4pt}\noindent\textbf{#1}}

\newcommand\blfootnote[1]{%
  \begingroup
  \renewcommand\thefootnote{}\footnote{#1}%
  \addtocounter{footnote}{-1}%
  \endgroup
}

\author{
  Yan Zhang$^{\dagger 2}$,\;  Sergey Prokudin$^{1,3}$,\; Marko Mihajlovic$^1$,\; Qianli Ma$^{\dagger4}$,\; Siyu Tang$^{1}$ \\
  $^1$ETH Z\"{u}rich, Switzerland , 
  $^2$Meshcapade \\
  $^3$ROCS, University Hospital Balgrist, University of Z\"{u}rich, 
  $^4$Nvidia \\ 
}

\begin{document}

\maketitle
\begin{abstract}

Understanding the dynamics of generic 3D scenes is fundamentally challenging in computer vision, essential in enhancing applications related to scene reconstruction, motion tracking, and avatar creation. In this work, we address the task as the problem of inferring dense, long-range motion of 3D points. By observing a set of point trajectories, we aim to learn an implicit motion field parameterized by a neural network to predict the movement of novel points within the same domain, without relying on any data-driven or scene-specific priors. To achieve this, our approach builds upon the recently introduced dynamic point field model~\cite{prokudin2023dynamic} that learns smooth deformation fields between the canonical frame and individual observation frames. However, temporal consistency between consecutive frames is neglected, and the number of required parameters increases linearly with the sequence length due to per-frame modeling. To address these shortcomings, we exploit the intrinsic regularization provided by SIREN~\cite{sitzmann2020implicit}, and modify the input layer to produce a spatiotemporally smooth motion field. Additionally, we analyze the motion field Jacobian matrix, and discover that the motion degrees of freedom (DOFs) in an infinitesimal area around a point and the network hidden variables have different behaviors to affect the model's representational power. This enables us to improve the model representation capability while retaining the model compactness. Furthermore, to reduce the risk of overfitting, we introduce a regularization term based on the assumption of piece-wise motion smoothness. Our experiments assess the model's performance in predicting unseen point trajectories and its application in temporal mesh alignment with guidance. The results demonstrate its superiority and effectiveness. The code and data for the project are publicly available\footnote{\label{footnote_supp}\url{https://yz-cnsdqz.github.io/eigenmotion/DOMA/}}.

\blfootnote{$\dagger$ This work was done when YZ and QM were at ETH Z\"{u}rich. }

\end{abstract}    

\begin{figure}
    \centering
    \includegraphics[width=\linewidth]{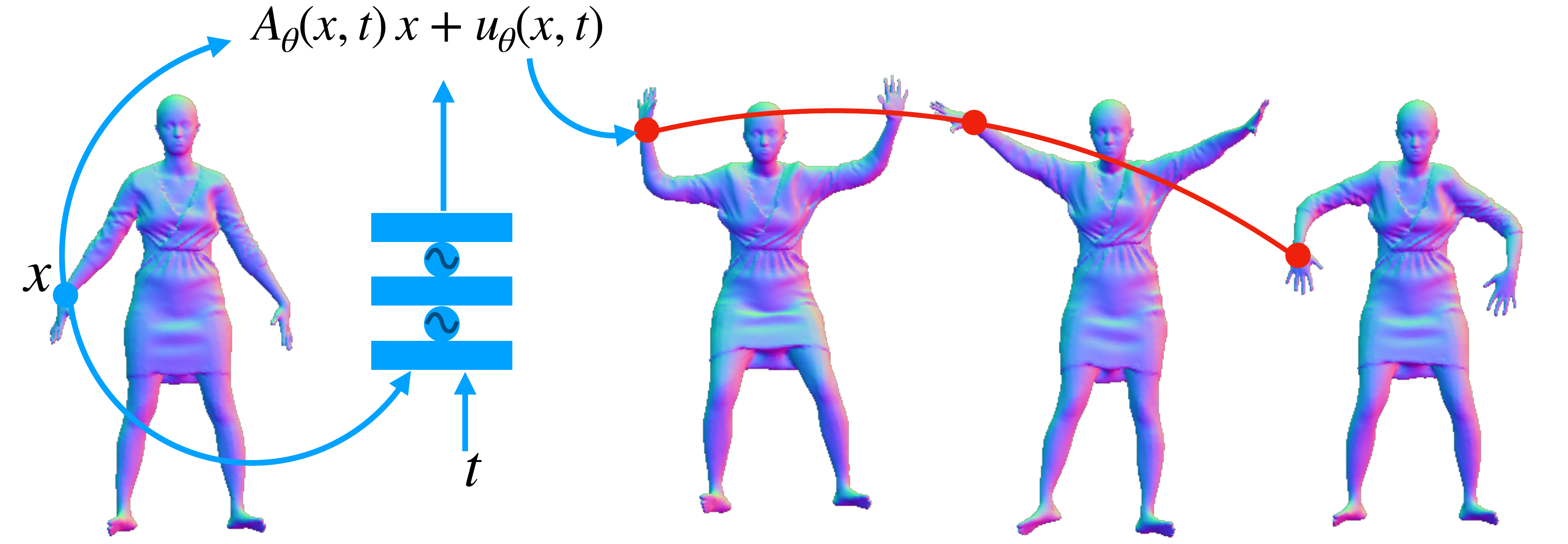}
    \caption{We introduce \textit{DOMA}, a compact implicit motion model designed to capture generic dynamics of 3D scenes. By processing a 3D point $\bm x$ in the canonical frame alongside a 1D time step $t$, DOMA predicts an affine mapping, parameterized by a linear map ${\bm A}_\theta$ and a translation vector ${\bm u}_\theta$. By leveraging the inherent regularity of the utilized SIREN framework~\cite{sitzmann2019siren}, DOMA ensures the generation of a spatiotemporally smooth motion field. The model’s capacity to represent complex dynamics can be controlled by adjusting the degrees of freedom of the output affine mapping.
    }
    \label{fig:teaser}
\end{figure}

\section{Introduction}
\label{sec:intro}

Motion estimation plays a crucial role in several key areas of computer vision, such as dynamic scene reconstruction, autonomous navigation, and avatar creation. Treated as a distinct task, it emerges in various contexts as non-rigid tracking~\cite{deng2022survey}, point set~\cite{tsin2004correlation,myronenko2010point} and mesh registration~\cite{bogo2014faust,groueix20183d}, shape matching~\cite{ovsjanikov2012functional}, as well as optical and scene flow estimation~\cite{zhai2021optical}. The solutions adopted in these contexts significantly differ based on the ultimate objective and the foundational assumptions about the scene. A substantial amount of research exists concentrating on human-centric~\cite{poppe2007vision,dfaust:CVPR:2017} and rigid object motion~\cite{behl2019pointflownet,wen2023bundlesdf}, alongside efforts in learning generic 2D motion priors in a data-driven manner~\cite{dosovitskiy2015flownet,karaev2023cotracker}. The diversity in applications and approaches underscores the complexity and significance of motion estimation within the field of computer vision~\cite{mathis2020primer}.

In this work, we aim to develop a motion model capable of reconstructing the dynamics of generic 3D scenes without relying on data-driven motion priors or object-specific models. Specifically, by analyzing observed point trajectories within dynamic 3D scenes, we seek to learn an implicit motion model capable of predicting the movement of novel 3D points. This task bears significant relevance to the process of warping 3D points across frames, a procedure commonly encountered in neural rendering~\cite{park2021nerfies,lombardi2019neural}, point cloud alignment~\cite{li2022non}, object tracking~\cite{wang2023omnimotion}, and avatar creation~\cite{ARAH:ECCV:2022,palafox2021npm,qian20233dgsavatar}. Typically, warping methods developed in these domains are intended to supplement the primary objectives such as the quality of novel view synthesis. Consequently, critical features such as the representational capability of the motion model, along with the consistency and plausibility of the motion it recovers, have not been the primary concern in these studies.

In this context, the closest work is the recently proposed dynamic point field (DPF) model~\cite{prokudin2023dynamic}, which addresses the task of recovering the correct implicit deformation function based on the observed pair of 3D surfaces. It proposes a lightweight deformation field formulated by a SIREN network~\cite{sitzmann2020implicit}, an MLP with periodic activation functions. Due to the regularity introduced by SIREN, its modeled deformation is spatially smooth, enabling various applications such as robust mesh registration and avatar animation.

Despite its effectiveness, the DPF method is limited to learning deformation fields between just two frames, the canonical and the target frames. To extend this to multiple frames, it proposes creating a set of deformation field models, where each model transforms points from the canonical frame to a distinct target frame. Consequently, the number of models increases with the sequence length, leading to significant memory and computational overhead. Moreover, the frame-wise approach fails to ensure temporal motion consistency, potentially resulting in discontinuities between consecutive frames and jittering artifacts. We conduct a thorough analysis of the representational capabilities of DPF and its underlying SIREN network to address these shortcomings. 
To achieve temporal smoothness, we follow the wave equation formulation in~\cite{sitzmann2020implicit}, and modify the input layer by incorporating a 1D time step along with the existing 3D point location input, similar to the strategy regularly implied in the implicit warping fields for neural rendering~\cite{pumarola2021d}.

Compared to the per-frame modeling scheme in DPF, this operation attempts to compress the entire sequence into a single network, which raises challenges on the model's capacity of motion representation.
Rather than increasing the number of hidden variables in the SIREN network, we opt to refine its output layer by introducing more motion degrees of freedom (DOFs).
From a mathematical standpoint of continuum mechanics~\cite{spencer2004continuum}, the advantages of additional DOFs are demonstrated in the Jacobian matrix of the motion field: Two points that are infinitely close to each other in space gain greater movement flexibility, provided the same number of network hidden variables.
Therefore, the model representation power is improved, and the model compactness is retained.

Nevertheless, additional DOFs can increase the risk of overfitting, in particular when the observed point trajectories are excessively sparse.
To overcome this issue, we leverage a generic assumption on motion, \ie piece-wise smoothness, and propose a motion smoothness term by penalizing the approximate L1 norm of spatial derivatives of the predicted transformations.
Here, rather than employing an auto-differentiation framework, we derive analytical gradients of our employed SIREN network to speed up the computation.

We undertake comprehensive experiments to validate the efficacy of our method. 
To assess its motion representation capabilities, we extract seven challenging sequences from the DeformingThings4D dataset~\cite{li20214dcomplete} and generate four synthetic sequences that exhibit basic 3D motions. Our method demonstrates consistently superior performance in predicting the motion of novel points, when compared to both state-of-the-art methods and their variations. 
Furthermore, we employ our technique in the task of temporal mesh alignment with guidance, and evaluate its performance on complex sequences from the Resynth dataset~\cite{ma2021power,ma2021scale}. 
Compared to the DPF baseline, our approach achieves comparable alignment accuracy, better temporal regularity, and significantly smaller models, occupying approximately 200KB versus 8MB in the saved checkpoint for a 30-frame sequence.

We refer to our approach as \textit{DOMA}, an acronym for \textbf{D}egrees \textbf{O}f freedo\textbf{M} m\textbf{A}tter, contending that additional degrees-of-freedom is essential to improve the expressivity of implicit motion models.
Technical contributions are summarized as follows:
\begin{itemize}
    \item We extend the state-of-the-art implicit model for surface deformation~\cite{prokudin2023dynamic} for continuous, multi-frame motion modeling, leading to an implicit, spatiotemporally smooth affinity field;
    \item We leverage the Jacobian matrix to analyze the motion field complexity, and discover that additional DOFs at the output layer improve the model representation capability while retaining the model compactness;
    \item To enhance the quality of the motion learned, we introduce a regularization term based on the piece-wise smoothness assumptions of the domain;
    \item We assess our model, demonstrating the benefits of various modeling decisions through experiments, on challenging long-term scene flow estimation and guided mesh alignment.
\end{itemize}

\section{Related Work}
\label{sec:related}

\myparagraph{Motion representation with object models.}
Given the motion of a set of points, it is a highly unconstrained problem to infer the motion of other points in their proximity. In many applications, such an inverse problem is solved based on an object model that performs as a strong prior of the dynamics.
Typical examples are marker-based human motion estimation~\cite{AMASS:ICCV:2019,loper2014mosh,zhang2021we}, or 3D pose estimation from imagery data~\cite{SMPL-X:2019,rempe2021humor,kanazawa2018end,zhang2021learning}, in which human parametric body models such as~\cite{SMPL-X:2019,xu2020ghum} are leveraged. 
Here, the bones of a body model serve as an intermediate proxy for all other points' motion: the trajectory of a point on the body surface is generated by the weighted average of the bone transformations, a technique referred to as linear blend skinning (LBS).
When extending the skinning weights to a vector field, as in~\cite{SCANimate:CVPR:2021,mihajlovic2021leap,chen2021snarf,wang2021metaavatar,ARAH:ECCV:2022,weng2022humannerf,COAP}, any point in the space can be animated by the bone transformations.
The same technology can be applied to animals~\cite{Zuffi:CVPR:2017}, babies~\cite{hesse2019learning}, humanoids~\cite{yuan2021simpoe}, and other categories.

In cases where the object model is not directly available, it can be jointly optimized together with the motion from  observations. 
This provides flexibility on the object categories to handle more generic dynamics.
For example, BANMO~\cite{yang2022banmo} proposes a generic deformable model, where a set of 3D Gaussians serve as the motion control proxy, analogous to bones. 
During optimization, the Gaussian locations and orientations are optimized together with their transformations.
Likewise, KeyTr~\cite{novotny2022keytr} proposes a bone basis to deform a point cloud across frames, in which the basis coefficients play a similar role as the skinning weights.

\myparagraph{Motion representation without object models.}
Another line of work models the motion of points without the reliance of intermediate proxies like bones. Instead, they represent the motion of all points in space as a dense field, in which each location stores a transformation matrix. The motion of a point will be determined by the transformations of its infinitesimal neighbourhood.
Methods under this paradigm are frequently employed in neural rendering~\cite{tewari2022advances,park2021nerfies,lombardi2019neural}, dense tracking~\cite{wang2023omnimotion}, surface reconstruction~\cite{niemeyer2019occupancy,palafox2021npm} and non-rigid geometry alignment~\cite{prokudin2023dynamic,li2022non}.
Niemeyer \etal~\cite{niemeyer2019occupancy} employ neural ODE~\cite{chen2018neural} to model the dynamics, and estimate the implicit occupancy function at the canonical frame and its evolution as time progresses.
Prokudin \etal~\cite{prokudin2023dynamic}, Pumarola \etal~\cite{pumarola2021d}, and Palafox \etal~\cite{palafox2021npm} leverage MLP to model a translation field (or scene flow field), and warp the point from the canonical frame to the target frame via addition.
Li \etal~\cite{li2021neural} employ a neural network to parameterize the flow field for regularization. The exploited network is a MLP with ReLU~\cite{krizhevsky2012imagenet} activation functions.
Park \etal~\cite{park2021nerfies} design a SE(3) transformation field, warping the points on the camera ray from the observation frame to the canonical frame, so as to train the neural radiance field~\cite{mildenhall2020nerf} reliably.
Lombardi \etal~\cite{lombardi2019neural} employs a mixture of scaled SE(3) warping fields for the purpose of neural rendering dynamic scenes.
Likewise, Li and Harada~\cite{li2022non} apply SE(3) or scaled SE(3) transformations to perform non-rigid point cloud alignment.
Compared to the SE(3) transformation, the scaled-SE(3) transformation is capable of representing the dilation or shrinking of an object. 
Going beyond points' locations, the spatial transformations can also be applied to features in neural networks and potentially improve the performance on \eg image classification~\cite{jaderberg2015spatial}.

\myparagraph{Relations to object shape and view recovery from images.}
Existing works such as~\cite{kanazawa2018learning,goel2020shape,gharaee2023self} study to learn neural models from an image collection, and recover the 3D shape in a canonical frame, the camera pose, and the texture of an object from a single image. 
Despite addressing different tasks, their solutions of composing the instance-level shape by the mean shape and deformation is relevant to our manner of motion modeling.
Furthermore, we are encouraged by these works to reconstruct dynamic scenes from multiview videos as future work.

\myparagraph{DOMA in context.}
Existing motion modeling approaches are developed together with individual applications, in which the network architectures, coordinate encodings, and other properties are diverse. 
The motion representation capabilities of their models are seldom investigated.
In contrast, we start with the basic assumption that the motion field has spatiotemporal regularity.
Therefore, we leverage the SIREN~\cite{sitzmann2019siren} network, and extend the start-of-the-art work DPF~\cite{prokudin2023dynamic} to a multi-frame, smooth affinity field model.
We leverage knowledge of continuum mechanics, exploit the Jacobian matrix to describe the motion field complexity, and find that DOFs at the output layer and the network hidden variables affect the model representation power in different manners.
Guided by these insights, we propose a solution to increase the model capacity while retaining the model lightweight.
Moreover, we introduce a smoothness regularization term to overcome overfitting, which does not assume the underlying motion is \eg rigid like in~\cite{park2021nerfies}.
The effectiveness of DOMA is demonstrated with experiments in Sec.~\ref{sec:exp}.

\section{Method}
\label{sec:method}

\subsection{Preliminaries}
\label{sec:method:pre}
\subsubsection{SIREN~\cite{sitzmann2020implicit}}
SIREN proposes an implicit neural representation, which is a multilayer perceptron (MLP) with periodic activation functions. Specifically, the MLP with $n+1$ layers is given by
\begin{equation}
    {\bm y} = {\bm W}_n \left(\phi_{n-1} \circ \phi_{n-2} \circ \cdots \circ \phi_{0} \right)({\bm x}) + {\bm b}_n,
\end{equation}
with $\phi_{i} = \sin({\bm W}_i {\bm x}_i + {\bm b}_i)$ and $i=\{0,1,...,n-1\}$.
The gradient of the model w.r.t. the input is another phase-shifted SIREN network, and hence is infinitely differentiable.
As reported in~\cite{sitzmann2020implicit}, the sinusoidal activation functions boost the model performance on the convergence speed, reconstruction quality, and smoothness, letting the MLP outperform baselines consistently and considerably.
However, this network requires special initialization to be trainable. Given ${\bm x} \in \mathbb{R}^d$, it is suggested to have the weights $w_i$ in the uniform distribution $\mathcal{U}(-\sqrt{6/d}, \sqrt{6/d})$, so that the model output will converge to a normal distribution.

\subsubsection{Dynamic Point Field (DPF)~\cite{prokudin2023dynamic}}
DPF proposes an implicit deformation field to model point dynamics, achieving state-of-the-art performance on surface reconstruction, geometry deformation, and avatar animation with challenging clothing.
Given a point ${\bm x} \in \mathbb{R}^3$ in the canonical frame, it learns a field ${\bm u}: \mathbb{R}^3 \to \mathbb{R}^3$ formulated by a SIREN network~\cite{sitzmann2019siren}, and then transforms the point to a new location ${\bm y}$, \ie 
\begin{equation}
\label{eq:dpf}
 {\bm y} = g({\bm x}) = {\bm x} + {\bm u}({\bm x}).
\end{equation}

To handle complex and rapid motion, an \textit{as-isometric-as-possible} (AIAP) loss term~\cite[Eq.13]{prokudin2023dynamic} is proposed to minimize changes of pair-wise distances between neighbor points during deformation.
Furthermore, it proposes guided geometry deformation via corresponding keypoints, which can avoid sub-optimal matching caused by the Chamfer distance~\cite{fan2017point}. 
Provided a set of keypoint pairs $\{ ({\bm v}_i^c, {\bm v}_i^t) \}_{i=1}^N$, and a pair of non-corresponding geometries (\eg meshes and point clouds) to align $({\bm M}_c, {\bm M}_t)$, the guided geometry deformation can be performed by minimizing

\begin{equation}
\label{eq:dpf_align}
    \alpha_1 \mathcal{L}_{CD} \left( {\bm M}_t, {\bm g}({\bm x}) \right) + \alpha_2\mathcal{L}_V ({\bm v}^c, {\bm v}^t ) + \alpha_3\mathcal{L}_{AIAP}\left({\bm g}({\bm x}), {\bm x}\right),
\end{equation}
in which ${\bm x} \in {\bm M}_c$, $\alpha$s denote the loss weights, and $\mathcal{L}_{CD}$, $\mathcal{L}_{V}$, and $\mathcal{L}_{AIAP}$ denote the Chamfer loss, the L1 loss, and the AIAP loss on the corresponding keypoints, respectively.
To align a sequence of geometries, DPF~\cite{prokudin2023dynamic} suggests learning a set of deformation fields that warp points in the canonical frame to individual target frames.

\subsection{DOMA: Spatiotemporal Affinity Motion Fields}

DOMA is an implicit motion field formulated by
\begin{equation}
\label{eq:affinity}
 {\bm y} = {\bm A}({\bm x}, t){\bm x} + {\bm u}({\bm x}, t),
\end{equation}
in which ${\bm x} \in \mathbb{R}^3$ is a point in the canonical frame, $t \in \mathbb{R}$ is the time step, ${\bm A}: \mathbb{R}^3 \times \mathbb{R} \to \mathbb{R}^{3\times 3}$ and ${\bm u}: \mathbb{R}^3 \times \mathbb{R} \to \mathbb{R}^{3}$ are formulated by a shared SIREN network.
Following Sitzmann et al.~\cite[Sec.5.4 of supp. mat.]{sitzmann2020implicit} that how the wave equation is formulated and solved, we incorporate the 1D time into SIREN as input, letting $\frac{\partial {\bm y}}{\partial t}$ be another phase-shifted SIREN and get regularized.

\subsubsection{On The Representation Power}
Different from the per-frame modeling mechanism of DPF~\cite{prokudin2023dynamic}, incorporating the 1D time domain into the input layer compresses the entire sequence into a single network, raising challenges on the model representation power.

Referring to~\cite{spencer2004continuum}, the DPF formula Eq.~\eqref{eq:dpf} is generic to model object deformation in continuum mechanics. However, it has limitations in empirical studies, motivating us to investigate the reasons. Rather than studying the entire domain, we look into an infinitesimal region around an arbitrary 3D point ${\bm x}$ in the canonical frame, and derive its Jacobian matrix as
\begin{equation}
\label{eq:dpf_jacob}
 \frac{\partial{\bm y}}{\partial{\bm x}} = {\bm I} + \nabla {\bm u}({\bm x}),
\end{equation}
which is the optimal linear approximation of the motion around ${\bm x}$ and $\nabla$ denotes the spatial gradient.

This Jacobian matrix is able to reflect the motion complexity.
Intuitively, $\frac{\partial{\bm y}}{\partial{\bm x}}$ indicates the difference of movements between two points that are infinitely close to each other. 
Without any constraints on $\nabla {\bm u}$, the model is capable of representing highly complex motion. 
However,  ${\bm u}$ is formulated by SIREN~\cite{sitzmann2020implicit}, letting $\nabla {\bm u}$ become to
\begin{equation}
\label{eq:dpf_jacob2}
    \nabla {\bm u} =  {\bm W}_n  \left(\prod_{i=0}^{n-1}{\bm W}_i \circ \varphi_{i}({\bm x}) \right) ,
\end{equation}
with $\varphi_{i} = \cos({\bm W}_i {\bm x}_i + {\bm b}_i)$. 
Due to $|\varphi_{i}| \leq 1 $, we can derive (see Sec.~\ref{sec:supp:theorem} in supp. mat.)
\begin{align}
\label{eq:dpf_jacob_bound}
    \|\nabla {\bm u}\|_{2}  & \leq d^\frac{n}{2} \left( \prod_{i=0}^{n} \| {\bm W}_i \|_2 \right)  \left( \prod_{i=0}^{n-1} \|{\varphi_{i}}({\bm x}) \|_2  \right) \\ & \leq d^n \cdot \prod_{i=0}^n \|{\bm W}_i\|_{2}
\end{align}
in which $\|\cdot \|_2$ is the L2 norm, \ie \textbf{the largest singular value} of the matrix and Euclidean norm of the vector.

Consequently, the movement difference between two neighboring points in the domain is constrained by Eq.~\eqref{eq:dpf_jacob} and Eq.~\eqref{eq:dpf_jacob_bound}.
To increase the representation power, one can straightforwardly increase the number of hidden layers, or the dimension of hidden variables, because both can increase the upper-bound of $\|\nabla {\bm u}\|_{2}$.

Our DOMA model can improve the model capacity without modifying the hidden layers.
Referring to Eq.~\eqref{eq:affinity}, its Jacobian matrix is given by
\begin{equation}
 \frac{\partial{\bm y}}{\partial{\bm x}} = {\bm A} + \langle \nabla {\bm A}, {\bm x} \rangle + \nabla {\bm u}({\bm x}),
\end{equation}
which replaces the identity matrix in Eq.~\eqref{eq:dpf_jacob} with two complex terms.
Since the identity matrix in Eq.~\eqref{eq:dpf_jacob} does not contribute to the motion complexity, our method intrinsically increases the complexity, while keeping the model hidden layers unchanged.
Consequently, more complex linear transformations with additional DOFs, such as scaling and shearing, are introduced to every infinitesimal area in the entire domain, thereby increasing the motion complexity globally.

\subsubsection{The Variants of DOMA}
The motion complexity can be controlled by applying different constraints on ${\bm A}$, leading to different versions according to the DOFs, inspired by existing works \eg~\cite{park2021nerfies,lombardi2019neural,li2022non}.
We denote the model Eq.~\eqref{eq:affinity} as \textbf{DOMA-Affinity}.
In our implementation, the SIREN network outputs a 12-dimensional variable. The first 9 variables are reshaped to ${\bm A}$ and the rests are regarded as ${\bm u}$.

\myparagraph{DOMA-Trans.}
When ${\bm A}$ is an identity matrix, Eq.~\eqref{eq:affinity} degenerates to a translation field, which formulated by
\begin{equation}
\label{eq:tdpf}
 {\bm y} = {\bm x} + {\bm u}({\bm x}, t).
\end{equation}
This model can be regarded as a straightforward extension of DPF~\cite{prokudin2023dynamic} to the spatiotemporal domain.

\myparagraph{DOMA-SE(3).}
With ${\bm A} = {\bm Q} \in SO(3)$, \ie a rotation matrix in the 3D space, Eq.~\eqref{eq:affinity} becomes to
\begin{equation}
\label{eq:se3}
 {\bm y} = {\bm Q}({\bm x}, t) {\bm x} + {\bm u}({\bm x}, t).
\end{equation}
Besides producing the translation, we let the SIREN network output the 6D continuous rotation representation~\cite{zhou2019continuity}, and perform orthogonalization to get the rotation matrix.

\myparagraph{DOMA-Scaled SE(3).}
By introducing an additional DOF for scaling, we can modify Eq.~\eqref{eq:affinity} to
\begin{equation}
\label{eq:scaledse3}
 {\bm y} = s({\bm x}, t) {\bm Q}({\bm x}, t) {\bm x} + {\bm u}({\bm x}, t),
\end{equation}
with $s({\bm x}, t)$ being the spatiotemporal scalar field.
In our implementation, we let the SIREN network to produce an additional 1D variable, and apply the softplus activation function~\cite{zheng2015improving,dugas2000incorporating} to product $s({\bm x}, t)$.

\subsubsection{Model Complexity Analysis}
By changing DOFs at the SIREN output layer, the hidden layers remain unchanged, and the model size is not linearly growing with the sequence length.
Provided the motion sequence has $T$ frames, DOMA-Affinity has $16d + nd^2$ parameters, with $n$ and $d$ denoting the number of hidden layers and the hidden dimensions, respectively.
In contrast, the per-frame modeling of DPF~\cite{prokudin2023dynamic} requires $T-1$ SIRENs and has $(6d+nd^2)(T-1)$ parameters in total, leading to $\mathcal{O}(T)$ model size.
Please see Tab.~\ref{tab:supp:complexity} in supp. mat. for more details.

\subsection{Motion Smoothness Regularization}

Although introducing extra DOFs can improve the model representation power, it increases the risk of overfitting. 
Without loss of generality, we assume the motion field is piece-wise smooth in the domain.
Referring to variational methods in optical flow \eg~\cite{zimmer2011optic,brox2004high}, we introduce the following smoothness regularization loss, \ie
\begin{equation}
    \label{eq:regh}
    \mathcal{L}_{H} = \mathbb{E}_{ t, {\bm x}} \left[ \Psi\left(\|\nabla {\bm A}\|_F^2 + \|\nabla {\bm u}\|_F^2 \right) \right],
\end{equation}
in which $\Psi(s^2) = \sqrt{1+s^2} -1$ is the convex Charbonnier function~\cite{charbonnier1994two} to approximate the L1 norm, $\nabla (\cdot)$ denotes computing the spatial gradients, and $\|\cdot\|_F$ denotes the Frobenius norm. 
Intuitively, the local motion is parameterized by ${\bm A}$ and ${\bm u}$, and hence the zero value of the above loss term suggests all points conduct the identical affine transformation.
The Charbonnier function plays the role of a robustifier, with which the difference of motions between two neighboring points is less penalized compared to L2 norm. 
With prior knowledge on the scene dynamics, $\Psi(s^2)$ can be changed to other terms, \eg $\Psi(s^2) = s^2$ to encourage homogeneous motion.

Instead of employing auto-differentiation tools in \eg PyTorch~\cite{paszke2019pytorch}, we implement analytical gradients of the SIREN network, \ie Eq.~\eqref{eq:dpf_jacob}, to speed up computation. See Tab.~\ref{tab:supp:runtime} in supp. mat. for an empirical study.

\section{Experiment}
\label{sec:exp}

Without explicit mentioning, we set the first frame in the sequence as the canonical frame, and normalize the time steps to $[-1,1]$. 
Please see supp. mat. for more details and additional experiments.

\subsection{Novel Point Motion Prediction}
\label{sec:exp1}
Based on a sparse set of observed point trajectories, we aim to predict the motion of unseen points during training, in order to verify the quality of learned dynamics.

\myparagraph{Datasets.}
We select 7 sequences with various object categories, shapes, and motions from the DeformingThings4D~\cite{li20214dcomplete} dataset.
For each sequence, we use 100 consecutive frames. 
We randomly select 25\% mesh vertices for training the motion field, and use the remaining ones for testing.
In addition, we create four synthetic sequences of elemental motions, \ie translation, rotation, scaling, and shearing, respectively, in order to investigate the model representation power in detail.
Each synthetic sequence has 20 frames and contains 3000 points uniformly sampled from $[-1,1]^3$. 
Likewise, we randomly select 25\% of points for training, and leave the remaining ones for testing.

\myparagraph{Evaluation metrics.}
For evaluation, we employ the learned motion field to transform testing points in the canonical frame to individual target frames.
We compute the scene flow end point error (EPE), \ie $ \mathbb{E}_{t \in \{1...,T\}}[\|{\bm v}_t - {\bm v}^{gt}_t\|_1]$ and ${\bm v}_t = {\bm y}_t - {\bm x}$, in which ${\bm y}_t$ denotes the estimated corresponding point of ${\bm x}$ at time $t$.
For DeformingThings4D, we additionally use the learned motion field to warp the canonical object mesh to each individual frame, sample $10^6$ points on both the warped mesh and the ground truth mesh, and compute the Chamfer distance~\cite{fan2017point} $\mathcal{L}_{CD}$ and the Chamfer normal distances $\mathcal{L}_n$, as in~\cite[Table 2]{prokudin2023dynamic}.

\myparagraph{Baselines and ours.}
This task is highly related to learning the warping fields in various scenarios, such as deformable object modeling~\cite{palafox2021npm}, scene flow estimation~\cite{li2021neural}, and neural rendering~\cite{pumarola2021d,park2021nerfies}.
The warping field is commonly parameterized with a neural network with ReLU activation functions~\cite{krizhevsky2012imagenet} and positional encodings~\cite{mildenhall2020nerf}.
Therefore, we leverage such kinds of neural networks as baselines. 
Specifically, we denote \textit{MLP-ReLU} as an MLP with ReLU activation functions and 6 hidden layers of 128 dimensions.
It takes the concatenation of the 3D location and the 1D time step as input and outputs motion vector.
Additionally, we introduce positional encoding~\cite{mildenhall2020nerf}, or output DCT coefficients~\cite{wang2021neural}, to create \textit{MLP-ReLU PE.6} and \textit{DCT-NeRF} as baselines, respectively.
Moreover, we adapt the \textit{BANMO}~\cite{yang2022banmo} deformation module into our setting, and implement a modified version named \textit{BoneCloud}, following the idea of~\cite{novotny2022keytr}. 
More details of these baseline methods are demonstrated in Sec.~\ref{sec:supp:exp_prediction} of supp. mat.

We denote the DOMA models with their suffixes, \ie \textit{-Trans}, \textit{-SE(3)}, \textit{-Scaled SE(3)}, and \textit{-Affinity}, respectively.
All their SIREN networks have 128 hidden dimensions and 2 hidden layers. 
Moreover, we implement the elastic regularization proposed in~\cite{park2021nerfies} to encourage rigid motion, which is denoted by \textit{-E}. Likewise, \textit{-H} denotes our motion smoothness regularization \textit{without} the Charbonnier function, encouraging the motion is homogeneous.

\begin{table}
    \centering
    \footnotesize
    \begin{tabular}{lccc}
    \toprule
         \textit{Methods} & \textit{EPE}$\downarrow$ & $\mathcal{L}_{CD}$ $\downarrow$ & $\mathcal{L}_{n}$ $\downarrow$  \\
         \midrule
         MLP-ReLU~\cite{mildenhall2020nerf} & 222.4 & 3.410 & 0.411\\
         MLP-ReLU PE.6~\cite{mildenhall2020nerf} &  237.7 & 3.747 & 0.431\\
         DCT-NeRF~\cite{wang2021neural} & 215.0 & 3.766 & 0.347\\
         BANMO~\cite{yang2022banmo} & 488.9 & 13.275 & 0.451\\
         BoneCloud~\cite{yang2022banmo,novotny2022keytr} & 136.1 & 1.993 & 0.261\\
         \midrule
         Ours-Trans & 78.5 & 1.401 & \textbf{0.215}\\
         Ours-SE(3) & 76.7 & 3.706 & 0.225\\
         Ours-Scaled SE(3) & \textbf{76.2} & 2.074 & 0.220\\
         Ours-Affinity & 78.1 & \textbf{1.266} & 0.218\\
        \bottomrule
    \end{tabular}
    \caption{Results on DeformingThings4D sequences. EPE and $\mathcal{L}_{CD}$ are in $\times 10^{-4}$. Best results are in boldface. }
    \label{tab:df4d}
\end{table}

\begin{table}
    \centering
    \footnotesize
    \begin{tabular}{lcccc}
    \toprule
         \textit{Methods} & Rotation & Scaling & Shearing & Translation  \\
         \midrule
         -Trans & 2725.4 & 1817.8 & 1619.5 & 1042.4 \\
         -SE(3) & {730.6} & 1991.4 & 1138.3 & 899.4 \\
         -Scaled SE(3) & {801.1} & {685.8} & 1524.7 & {1096.2} \\
         -Affinity & 1486.0 & {915.4} & {622.1} & 822.4\\
         \midrule
         -Trans-E & 38.0 & 1669.6 & {753.6} & 38.8 \\
         -SE(3)-E & 20.0 & 1761.3 & 832.7 & 26.4 \\
         -scaled SE(3)-E & 21.2 & {1161.8} & 961.1 & 24.0 \\
         -Affinity-E & {19.2} & 155.7 & 864.0 & {15.7} \\
         \midrule
         -Trans-H & 4919.9 & 2056.4 & 2446.8 & 37.8 \\
         -SE(3)-H & 52.4 & 2012.4 & 1665.0 & 36.9 \\
         -scaled SE(3)-H & 29.3 & 22.1 & 688.0 & 30.3 \\
         -Affinity-H & \textbf{5.4} & \textbf{26.3} & \textbf{8.5} & \textbf{28.8} \\
        \bottomrule
    \end{tabular}
    \caption{Results on Synthetic sequences w.r.t. EPE (in $\times 10^{-4}$). `-E' denotes the elasticity loss proposed in Nerfies~\cite{park2021nerfies}, and `-H' denotes our smoothness loss. }
    \label{tab:synthetic}
\end{table}

\begin{figure}
    \centering
    \includegraphics[width=0.99\linewidth]{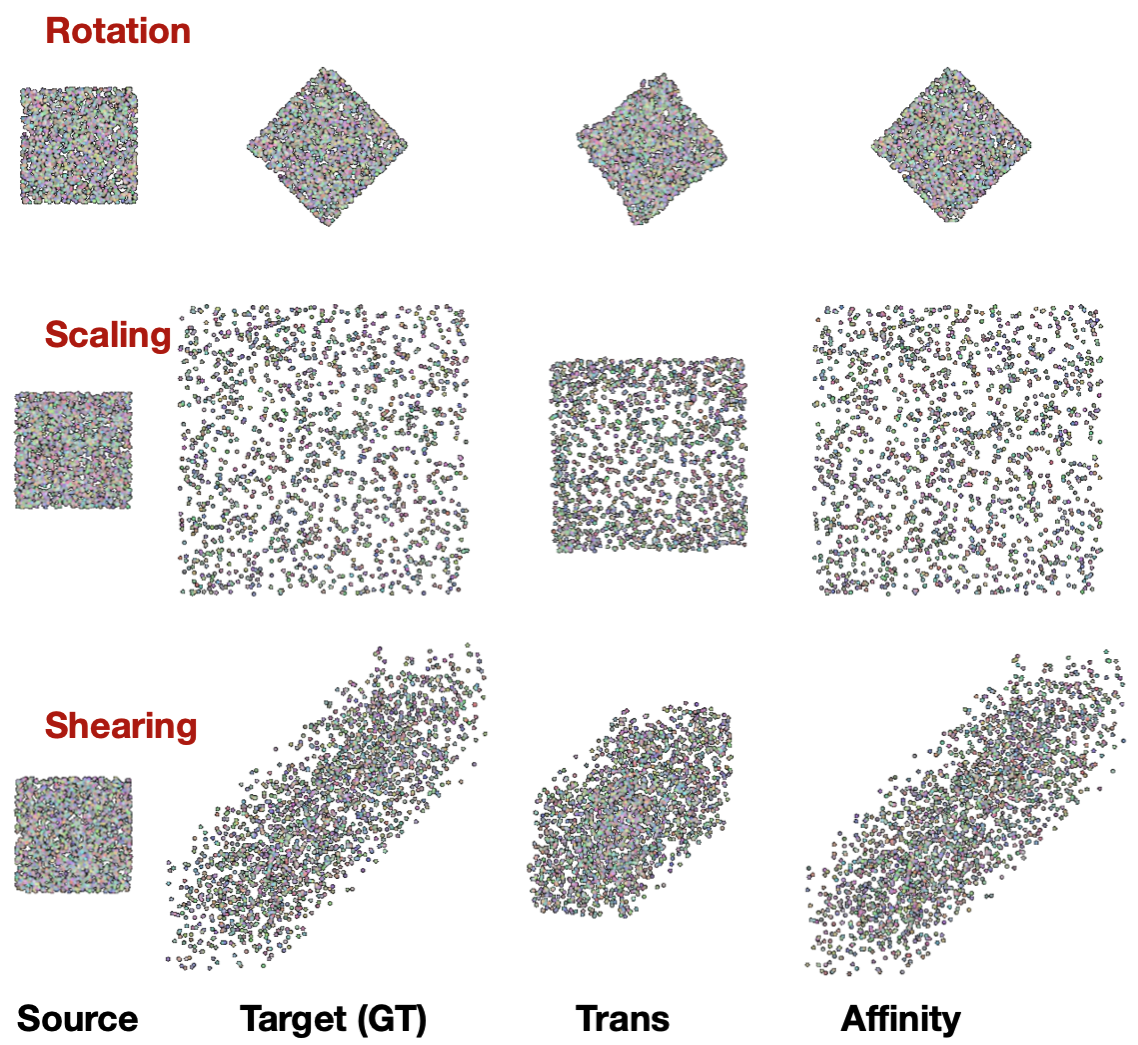}
    \caption{Qualitative results on the Synthetic sequences. The smoothness regularization is applied. Rows show types of motions, and columns show the testing points in the canonical frame, a target frame, and estimated results, respectively. See more descriptions in Sec.~\ref{sec:exp1} and Sec.~\ref{sec:supp:more_on_synthetic} in supp. mat.}
    \label{fig:synthetic}
\end{figure}

\myparagraph{Results.}
The results on the DeformingThings4D sequences are presented in Tab.~\ref{tab:df4d}.
We can see that the SIREN-based methods achieve comparable performance, and outperform other baseline methods consistently.

The results on the synthetic sequences are shown in Tab.~\ref{tab:synthetic} and Fig.~\ref{fig:synthetic}.
First, we can see that the incorporated DOFs lead to different motion representation behaviors.
Methods with SE(3) transformations are more effective in representing rotation and translation. 
By comparing `-Scaled SE(3)' and `-SE(3)', we can see the additional DOF benefits modeling scaling.
The affine transformation is effective in representing all cases.
Second, we can see appropriate regularization is important. 
The elasticity regularization proposed in~\cite{park2021nerfies} is effective to encourage rigid motions, but performs inferior if the motion is non-rigid, \ie scaling and shearing. 
On the other hand, our smoothness regularization significantly boosts the performance, if the sequence fits the modelled DOFs.

Moreover, we find the structure of the linear matrix ${\bm A}$ and the network hidden dimension have different behaviors to affect the model representation power. 
Increasing the hidden dimension cannot simply introduce new DOFs that the model can represent. see Sec.~\ref{sec:supp:more_on_synthetic} in supp. mat. for details.

\subsection{Guided Mesh Alignment}
\label{sec:exp2}
Aligning a mesh template to a sequence of scans is important in various applications, such as graphics~\cite{ma2021scale} and healthcare~\cite{hesse2019learning}.
We follow the guided geometry deformation method of DPF~\cite{prokudin2023dynamic}.
Referring to Eq.~\ref{eq:dpf_align}, we minimize

{
\scriptsize
\begin{equation}
    \label{eq:gmaloss}
    \sum_t \{\alpha_1\mathcal{L}_{CD} \left( {\bm M}_t, {\bm y}({\bm x}) \right) + 
    \alpha_2\mathcal{L}_V ({\bm v}^c, {\bm v}^t ) 
    + \alpha_3\mathcal{L}_{AIAP}\left({\bm y}({\bm x}), {\bm x}\right) + \alpha_4\mathcal{L}_H  \} .
\end{equation}
}

In this experiment, we set $(\alpha_1, \alpha_2) = (10^3, 1)$. The regularization loss weights are $(\alpha_3, \alpha_4) = (1, 0.001)$ if they are enabled. Training terminates after 2000 iterations for all cases.

\myparagraph{Datasets.}
We employ the Resynth dataset~\cite{ma2021scale,ma2021power}, in which the human bodies perform articulate motions, while the clothing, in particular the long skirt, moves accordingly in a highly complex manner.
We use 4 sequences from 4 individual subjects with different genders, body shapes, and clothing types. Each sequence is downsampled by every 2 frames, and afterwards the first 30 frames are selected, leading to 16 sequences with 480 frames in total. 
We use the SMPL-X~\cite{SMPL-X:2019} mesh vertices (10,475 points per frame) as the guidance points, and the low-resolution scans (40,000 points per frame) as the targets to fit. 
Furthermore, we perform Poisson surface reconstruction on the low-res scan at the canonical frame, and obtain a mesh template with about 60K vertices and 130K faces. 
During training, we learn the motion field based on the guidance points and the low-res scans, and minimize Eq.~\eqref{eq:gmaloss}.
During testing, we warp the mesh template vertices to individual frames based on the learned motion field, and re-compute the vertex normals. 
Note these mesh template are unseen during training.

\myparagraph{Evaluation metrics.}
For evaluation, we compute the Chamfer distance of the vertex locations and normals between the warped meshes with the target low-res scans, \ie $\mathcal{L}_{CD}$ and $\mathcal{L}_n$, as in Tab.~\ref{tab:df4d} as well as DPF~\cite[Table 2]{prokudin2023dynamic}. 
To verify the temporal smoothness, we additionally compute two metrics: 
1) The standard deviation (std) of edge lengths along the temporal dimension, and report its maximal value. This metric is able to reflect whether the mesh is significantly stretched or not.
2) The averaged std of the velocity l2-norm of the mesh vertices, which measures the temporal smoothness.
These two metrics are denoted as \textit{STD(E)} and \textit{STD(V)}, respectively.
Their values are the lower the better, but should not vanish because of the conducted deformation and motion.

\myparagraph{Baselines and ours.}
We compare the frame-wise DPF scheme that is suggested by~\cite{prokudin2023dynamic}.
In addition, we investigate the effectiveness of the regularization loss terms AIAP~\cite[Eq.13]{prokudin2023dynamic} and our proposed motion smoothness term.
The method notations are the same with Sec.~\ref{sec:exp1}.
All SIREN networks have 3 hidden layers and 128 hidden dimensions.

\begin{figure*}
    \centering
    \includegraphics[width=\linewidth]{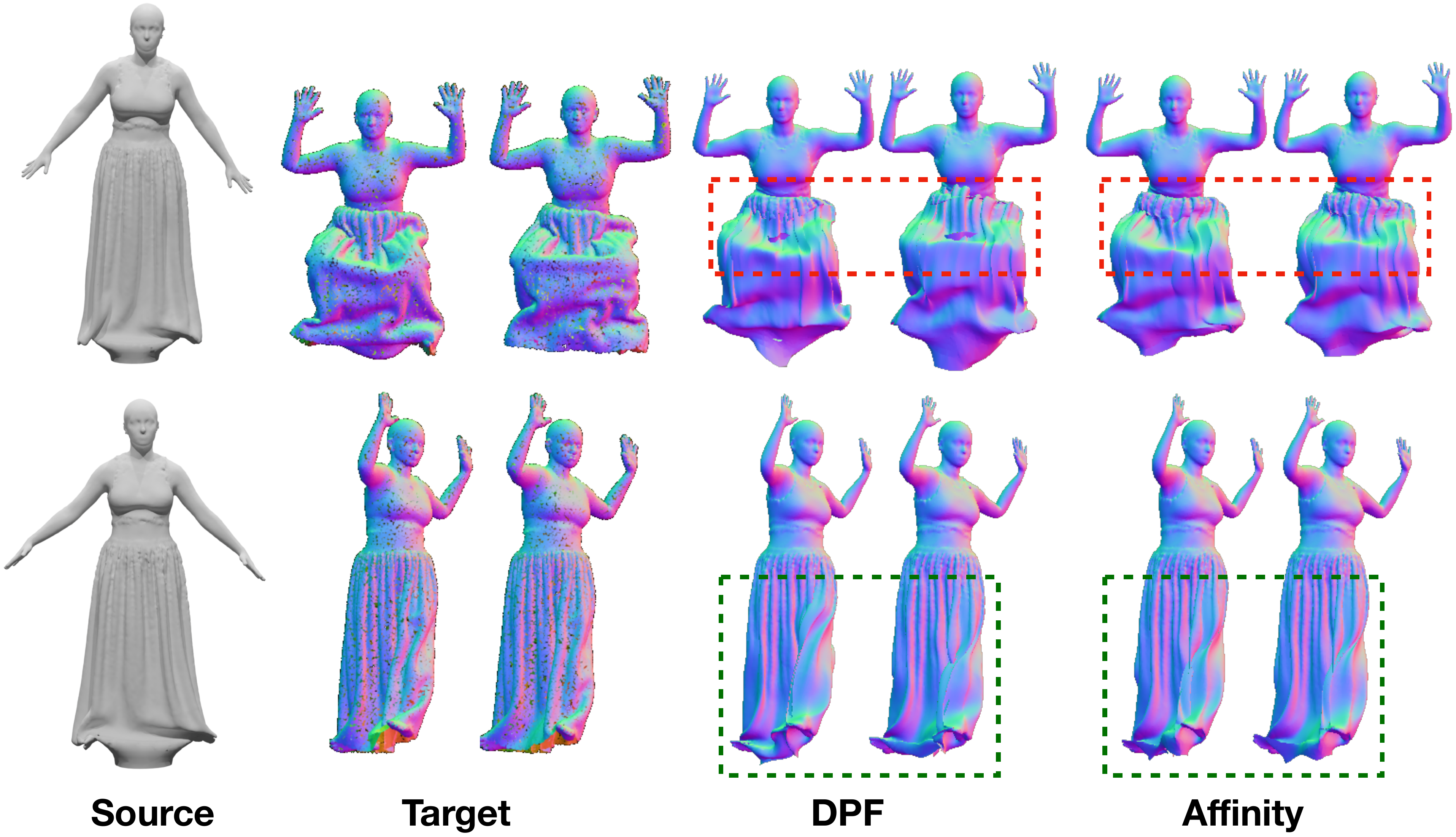}
    \caption{Illustration of results on two Resynth sequences. From left to right: The source mesh in the canonical frame, two consecutive frames of the target scans, the results from frame-wise DPF and DOMA-Affinity, respectively. Both AIAP~\cite{prokudin2023dynamic} and smoothness regularization are applied. The bounding boxes highlight significant changes. }
    \label{fig:resynth}
\end{figure*}

\begin{table}
    \centering
    \footnotesize
    \begin{tabular}{lcccc}
    \toprule
          & $\mathcal{L}_{CD}\downarrow$ & $\mathcal{L}_n\downarrow$ & \textit{STD(E)}$\downarrow$  & \textit{STD(V)}$\downarrow$   \\
           \midrule
         DPF~\cite{prokudin2023dynamic} & 1.149 & \textbf{0.122} & \textbf{11.6} & 24.6\\
         -Trans & 1.230 & 0.128 & 12.8 & 22.9  \\
         -Affinity & \textbf{1.142} & 0.125 & 11.9 & \textbf{22.8} \\
         \midrule
         DPF-A~\cite{prokudin2023dynamic} & 1.166 & \textbf{0.119} & \textbf{10.3} & 24.2\\
         -Trans-A & 1.195 & 0.123 & 10.4 & \textbf{23.0}  \\
         -Affinity-A & \textbf{1.151} & 0.122 & 10.6 & \textbf{23.0} \\
         \midrule
          DPF-H~\cite{prokudin2023dynamic} & 1.142 & \textbf{0.123} & 10.3 & 24.2\\
         -Trans-H & 1.207 & 0.128 & 10.8 & \textbf{22.9}  \\
         -Affinity-H & \textbf{1.127} & 0.127 & \textbf{10.1} & \textbf{22.9} \\
         \midrule
         DPF-AH~\cite{prokudin2023dynamic} & 1.189 & \textbf{0.120} & 9.3 & 24.3\\
         -Trans-AH & 1.240 & 0.124 & 9.3 & \textbf{23.0}  \\
         -Affinity-AH & \textbf{1.187} & 0.124 & \textbf{8.9} & \textbf{23.0} \\
        \bottomrule
    \end{tabular}
    \caption{Results of guided mesh alignment on our selected Resynth sequences. $\mathcal{L}_{CD}$ is in $\times 10^{-4}$. \textit{STD(E)} and \textit{STD(V)} are given in millimeters. `-A' and `-H' denote the AIAP regularization~\cite{prokudin2023dynamic} and our smoothness regularization, respectively. `-AH' denotes both regularization terms are applied. Best results are in boldface. Please see Tab.~\ref{tab:supp:resynth_all} for the performance of all models.}
    \label{tab:resynth}
\end{table}

\begin{table}
    \centering
    \footnotesize
    \begin{tabular}{lcc}
    \toprule
          & \#params. & checkpoint size (KB)   \\
           \midrule
          DPF~\cite{prokudin2023dynamic} & 1497600 & 7800 \\
         \midrule
         -Trans & 50048 & 139.6   \\
         -SE(3) & 50816 & 209.2  \\
         -Scaled SE(3) & 50944 & 209.8  \\
         -Affinity & 51200 & 210.8\\
         
        \bottomrule
    \end{tabular}
    \caption{Evaluations on the model size on the Resynth sequence. Since lightweight models are preferred, the numbers here are the lower the better. }
    \label{tab:dpfbm}
\end{table}

\myparagraph{Results.}
The results are shown in Tab.~\ref{tab:resynth}.
Compared to frame-wise DPF, we can see DOMA-Trans leads to consistently worse alignment accuracy, but better temporal smoothness.
This indicates the temporal regularity is obtained by compressing the entire motion into a single SIREN-based network, whereas the model's representational power is not sufficient.
When replacing the translation field by an affinity field, \ie DOMA-Affinity, the alignment accuracy is consistently improved to a similar level with frame-wise DPF, and the temporal smoothness is retained.
This indicates that introduced extra DOFs can effectively improve the model representation power.
In addition, the AIAP loss and the smoothness regularization can individually improve the performances, but their combination does not lead to obvious advantages, except for the edge length variations.
Fig.~\ref{fig:resynth}~\footnote{Quantitatively, \ie ($\mathcal{L}_{CD}$, $\mathcal{L}_{n}$, STD-E, STD-V) as in Tab.~\ref{tab:resynth}, DPF gives $(3.962,0.229,16.0,35.7)$ and $(2.557,0.175,13.8,18.5)$ for the top and bottom rows, respectively, whereas DOMA-Affinity gives $(3.208,0.216,19.8,29.3)$ and $(2.532,0.169,10.5,15.6)$.} illustrates some pairs of consecutive frames. 
We can see that the frame-wise DPF scheme causes visible discontinuities and artifacts in some regions, whereas the results of the affinity field have better temporal regularity.

Furthermore, the advantage of DOMA can be reflected by the model compactness. 
As shown in Tab.~\ref{tab:dpfbm}, DOMA models are significantly more lightweight. 
Adding additional DOFs at the network output layer only increases the number of parameters marginally.
\section{Conclusion}

In this work, we have advanced the DPF framework~\cite{prokudin2023dynamic} into a continuous, multi-frame affinity field model, which inherently exhibits spatiotemporal regularity and improves representational capabilities without compromising compactness.
Incorporating the 1D time domain to the network input layer ensures temporal regularity, and the DOFs at the output layer can manipulate the model representation power without modifying the network hidden variables.
The experimental results on novel point motion prediction and guided mesh alignment show its effectiveness and superiority to baselines.

\myparagraph{Limitations and future works.}
First, we have 4 loss terms to minimize in the task of guided mesh alignment, and inappropriate loss weights can degrade the performance considerably. How to balance their weights is still not transparent, which is worthy exploring in the future.
Second, our method can be employed to model warping fields for dynamic scene reconstruction and rendering, which is not covered in this paper and will be studied as future work.
Note that our method requires corresponding points between frames. Additional DOFs are effective to represent fine-grained movements, but might behave as a disadvantage to extract corresponding points due to less constraints.
Third, our advanced model representation power is potential to model highly complex dynamics, \eg fluid fields, which can benefit specific applications of medicine, aerodynamics, physics, \etc
Furthermore, our model is deterministic and does not consider the motion uncertainty. Thus, a future direction is to develop a generative model on dynamics, which synthesizes diverse dynamics based on the same set of point trajectories.

\section*{Acknowledgement}
We sincerely thank Shaofei Wang for fruitful suggestions, discussions, and other helps. 
This project is partially supported by the SNSF grant 200021 204840.

\newpage

{
    \small
    \bibliographystyle{ieeenat_fullname}
    \bibliography{main}
}

\clearpage

\setcounter{table}{0}
\setcounter{figure}{0}
\renewcommand{\thetable}{A\arabic{table}}
\renewcommand{\thefigure}{A\arabic{figure}}

\maketitlesupplementary

{
  \hypersetup{linkcolor=blue}
  \tableofcontents
}

\newpage

\section{Discussions on the Motion Model Bound} 
\label{sec:supp:theorem}

The singular values of the 3D linear transformation matrix carry essential physical meanings.
Via singular value decomposition, the motion can be regarded as a consecutive operations of rotating to a new coordinate frame, performing scaling in each dimension based on the singular values, and rotating back to the original coordinate frame.
Therefore, the upper bound of such deformation is indicated by the largest singular value.

In the main paper, we demonstrate the representation power of DPF~\cite{prokudin2023dynamic} is bounded, based on Eq.~\eqref{eq:dpf_jacob2} and Eq.~\eqref{eq:dpf_jacob_bound}.
Here we present more details in terms of a theorem with proof.

\begin{theorem}
Provided 
\begin{equation}
    \nabla {\bm u} =  {\bm W}_n  \left(\prod_{i=0}^{n-1}{\bm W}_i \circ \varphi_{i}({\bm x}) \right) 
\end{equation}
and
\begin{equation}
    \varphi_{i} = \cos({\bm W}_i {\bm x}_i + {\bm b}_i),
\end{equation}
in which $\circ$ is the composition of element-wise multiplication and broadcasting a vector to a matrix, as well as $i=\{0,1,...,n-1\}$, 
the bound of the spectral norm of $\nabla {\bm u}$ is given by

\begin{equation}
    \|\nabla {\bm u}\|_{2} \leq d^{n} \cdot \prod_{i=0}^n \|{\bm W}_i\|_{2},
\end{equation}
in which $n$ and $d$ denote the number of hidden layers and the dimension of hidden layers, respectively.

\end{theorem}

\begin{proof}
Referring to the matrix norm properties~\cite{matnorm,higham2002accuracy},
we have the following inequalities on the spectral norm, \ie
\begin{align}
\label{eq:bound_proof}
    \|\nabla {\bm u}\|_2  &=  \left\|{\bm W}_n  \left(\prod_{i=0}^{n-1}{\bm W}_i \circ \varphi_{i}({\bm x}) \right) \right\|_2 \\
    & \leq \| {\bm W}_n \|_2 \cdot \prod_{i=0}^{n-1} \| {\bm W}_i \circ \varphi_{i}({\bm x}) \|_2 \\
    & \leq \| {\bm W}_n \|_2 \left( \prod_{i=0}^{n-1} \| {\bm W}_i \|_2 \right)  \left( \prod_{i=0}^{n-1} \|\hat{\varphi}_{i}({\bm x}) \|_2  \right) \\
    & \leq \left( \prod_{i=0}^{n} \| {\bm W}_i \|_2 \right)  \left( \prod_{i=0}^{n-1} \|\hat{\varphi_{i}}({\bm x}) \|_2  \right),
\end{align}
in which $\hat{\varphi}_i$ is the matrix with the same column $\varphi_i$, and has the shape of $\mathbb{R}^{d \times q}$. 
This corresponds to the shape of ${\bm W}_i$, and hence it has $q=3$ at the input layer and $q=d$ in the hidden layers. Therefore, we assume $q=d$ in the following derivations to obtain the upper bound.

Note the rank of the matrix $\hat{\varphi}_i$ is 1, and we can have 
\begin{align}
 \|\hat{\varphi_{i}}({\bm x}) \|_2 = \|\hat{\varphi_{i}}({\bm x}) \|_F & = \sqrt{d} \cdot \|\varphi_{i}({\bm x}) \|_2 \leq d,
\end{align}
according to $\|\varphi_{i}\|_{\infty} \leq 1$.
Thus, we can derive
\begin{equation}
    \|\nabla {\bm u}\|_{2} \leq d^{n} \cdot \prod_{i=0}^n \|{\bm W}_i\|_{2}.
\end{equation}
\end{proof}

Although the constant factor $d^{n}$ is large, it is only reached when every $|\varphi_i|$ is equal to 1, which is implausible in practice. In addition, the entries in ${\bm W}_i$ are from the uniform distribution with a tiny range around 0~\cite{sitzmann2020implicit}, which further constrains the spectral norm of the Jacobian matrix.
Due to the challenges of spectral analysis on high-dimensional random matrices, we can look into the degenerated 1D case, which is given by
\begin{equation}
    \frac{du}{dx} = w_n \prod_{i=0}^{n-1} w_i \cos(w_i x + b_i).
\end{equation}
In this case, we can easily derive
\begin{align}
    \left|\frac{du}{dx} \right| = \left|w_n \prod_{i=0}^{n-1} w_i \cos(w_i x + b_i) \right| \leq \prod_{i=0}^{n} |w_i|,
\end{align}
which indicates that the motion complexity is heavily bounded.

\paragraph{A statistical perspective.} 
The boundedness can be also investigated from a statistical perspective. Starting with
\begin{equation}
    \|\nabla {\bm u}\|_2 \leq \| {\bm W}_n \|_2 \cdot \prod_{i=0}^{n-1} \| {\bm W}_i \circ \varphi_{i}({\bm x}) \|_2
\end{equation}
that is from Eq.~\eqref{eq:bound_proof}, we can reason the entries of ${\bm W}_i \circ \varphi_{i}({\bm x})$ are converging to the standard normal distribution if the model weights are initialized as in SIREN~\cite{sitzmann2020implicit}. 
Specifically, the entries of ${\bm W}_i$ are from the defined uniform distribution, $\varphi_i$ is from the arcsine distribution, since the cosine activation function is equivalent to the phase-shifted sine activation function and the bias does not modify the distribution for high enough frequency~\cite[Theorem 1.8]{sitzmann2020implicit}.
According to~\cite[Theorem 2.5]{chafai2009singular}, the largest singular value of ${\bm A}_i = {\bm W}_i \circ \varphi_{i}({\bm x})$ is bounded, having 
\begin{equation}
    \lim_{d \to \infty} \sup \lambda_1(d^{-1}{\bm A}_i{\bm A}^T_i) \leq 4,
\end{equation}
in which $d$ is the hidden dimension and $\lambda_1$ is the largest eigenvalue.
Therefore, their compositions with $i=0,\dots,n-1$ are also bounded.

\section{Additional Discussions on DOMA}

\subsection{The Network}

\paragraph{The model architecture.}
The DOMA models can be visualized in Fig.~\ref{fig:supp:doma}. In this case, the SIREN~\cite{sitzmann2020implicit} network contains one input layer, one output layer, and two hidden layers.
In the case of the `SE(3)' and `scaled SE(3)' models, 6D continuous rotation representations~\cite{zhou2019continuity} are produced by the output layer, which are then orthogonalized to rotation matrices.

\begin{figure}
    \centering
    \includegraphics[width=\linewidth]{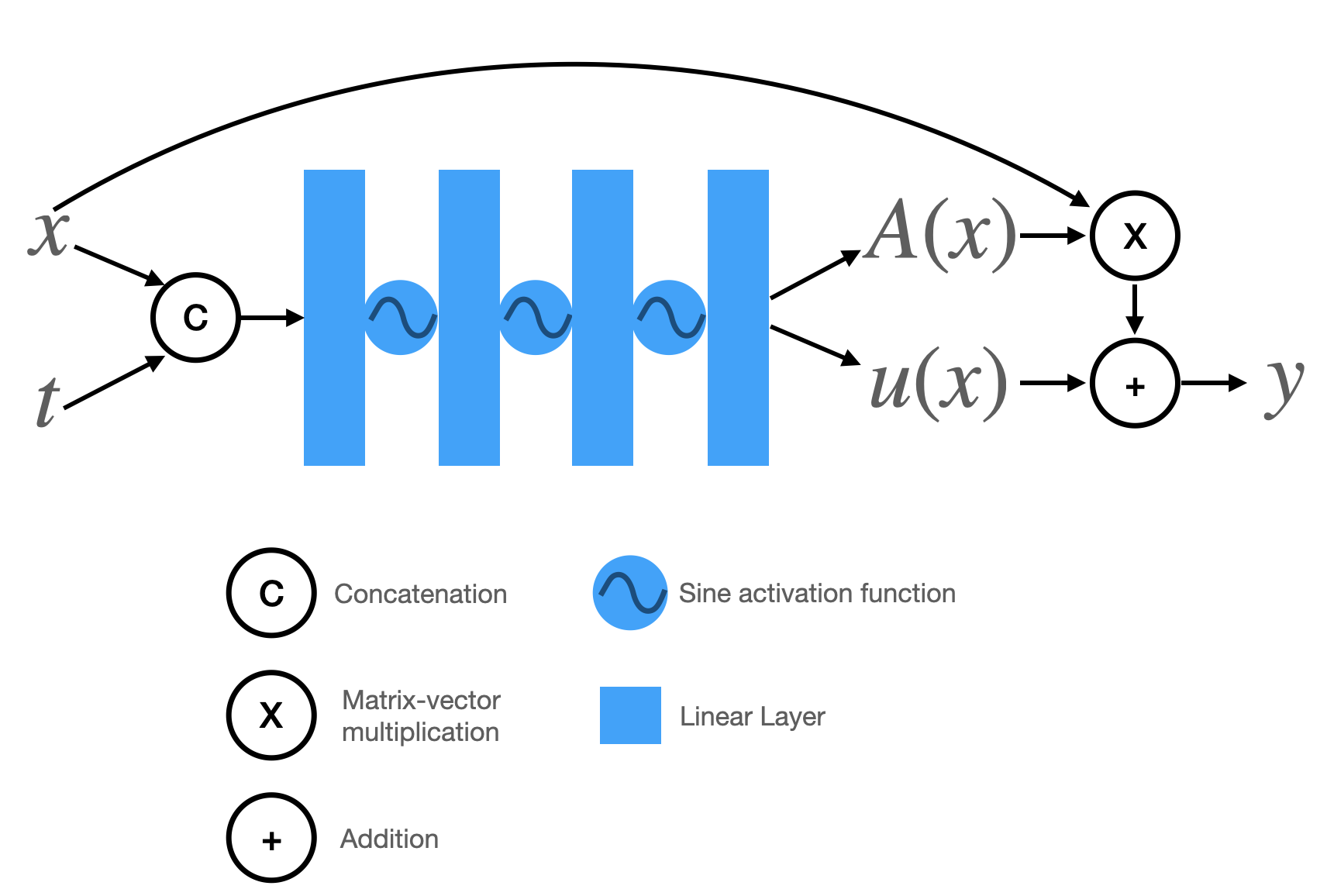}
    \caption{Illustration of the DOMA model architecture. The SIREN layers~\cite{sitzmann2020implicit} produce an affine transformation, which maps the point from ${\bm x}$ to ${\bm y}$ at time $t$.}
    \label{fig:supp:doma}
\end{figure}

\paragraph{The model sizes.}
Since the 1D temporal dimension is incorporated in the input layer, DOMA models have $\mathcal{O}(1)$ complexity w.r.t. the motion sequence length. The sizes of different models are summarized in Tab.~\ref{tab:supp:complexity}.

\begin{table}
    \centering
    \footnotesize
    \setlength{\tabcolsep}{4pt}
    \scalebox{0.90}{
    \begin{tabular}{ccccc}
    \toprule
         \textit{DPF} & \textit{-Trans} & \textit{-SE(3)} & \textit{-Scaled SE(3)} & \textit{-Affinity}  \\
         $(6d+nd^2)(T-1)$ & $7d+nd^2$ & $13d+nd^2$ & $14d+nd^2$ & $16d+nd^2$ \\
        \bottomrule
    \end{tabular}}
    \caption{The number of parameters in the employed SIREN network. $T$, $n$, $d$ denote the number of frames in the sequence, the number of network hidden layers, and the hidden dimension, respectively. The suffixes denote different versions of DOMA.}
    \label{tab:supp:complexity}
\end{table}

\subsection{Additional Discussions on Novelties}

Modeling the deformation field or the motion field is not a new task.
Instead, various methods have been developed within respective tasks, such as geometry deformation, neural rendering, dynamic scene reconstruction, avatar creation, etc.
Their exploited motion methods are diverse in terms of the neural architecture, positional encoding, underlying deformation models, and so on.
However, an important aspect is often overlooked: the motion field should be spatiotemporally regularized by nature.
To fill this gap,
we leverage the SIREN~\cite{sitzmann2019siren} network, and extend the start-of-the-art work DPF~\cite{prokudin2023dynamic} to a multi-frame smooth affinity field model.
By introducing additional DOFs at the output layer, we find the model representation power is improved in a different way from enlarging the model hidden layers, and come up with a solution to increase the model capacity while retaining the model size. 
Moreover, we introduce a smoothness regularization term to overcome overfitting, which does not assume the underlying motion is \eg rigid like in~\cite{park2021nerfies}.
The effectiveness of DOMA is demonstrated with experiments in Sec.~\ref{sec:exp} and the supp. mat.

The advantage of DOMA is more obvious when the ground truth motion is more complex.
An example is modeling the loose long skirt motion.
As shown in Tab.~\ref{tab:supp:persubject}, DOMA-Affinity is consistently superior to DPF on the `felice' sequences of Resynth. 
Another example is modeling the fluid dynamics, which is investigated in Sec.~\ref{sec:supp:fluid}. We can see DOMA-Affinity outperforms DPF significantly. 
Since DPF only models deformations between the canonical frame and frame $t$, it cannot ensure the temporal smoothness between $t$ and $t+1$.

Despite aiming at different tasks, our work is also related to object shape and view recovery from images.
Kanazawa \etal~\cite{kanazawa2018learning} propose a framework to learn from an annotated image collection, and recover the 3D shape in a canonical frame, the camera pose, and the texture of an object from a single image. 
The 3D object shape is parameterized by a learned mean shape and per-instance predicted deformation.
To encourage additional properties such as surface smoothness and regularized deformation, generic priors are leveraged in the training loss.
Goel \etal~\cite{goel2020shape} extend this framework to learn from an image collection without annotations of the keypoints and the camera. 
To further improve the performance, Gharaee \etal~\cite{gharaee2023self} propose to predict a set of keypoints to represent the shape, corresponding to positions on the category-specific mean shape in 3D. Afterwards, the camera pose is estimated by a robust PnP network~\cite{campbell2020solving}. 
These solutions of decoupling the instance-level shape into the mean shape and the deformation also inspire us how to model motions.
Furthermore, we are encouraged by these works to reconstruct dynamic scenes from multiview videos as future work.

\section{Experiment Details}

\subsection{Additional Presentations on Point Motion Prediction (Sec.~\ref{sec:exp1})}
\label{sec:supp:exp_prediction}
We leverage and modify the codebase of ResFields~\cite{mihajlovic2023resfields} for the implementations of baselines MLP-ReLU and DCT-NeRF~\cite{wang2021neural}.

\subsubsection{Dataset}
The 7 sequences from DeformingThings4D~\cite{li20214dcomplete} are listed in Tab.~\ref{tab:supp:dft4d}.
For each sequence, we extract the first 100 frames and regard the first frame as the canonical frame.

\begin{table}[t!]
    \centering
    \begin{tabular}{l}
    \toprule
           bear3EP\_Agression \\
           demon\_JazzDancing \\
           dragonOLO\_act25 \\
           michelle\_StepHipHopDance\\
           mutant\_Defeated\\
           tigerD8H\_Swim17\\
           vampire\_Breakdance1990\\
           vanguard\_JoyfulJump\\
        \bottomrule
    \end{tabular}
    \caption{The leveraged DeformingThings4D~\cite{li20214dcomplete} sequences in Sec.~\ref{sec:exp1}.}
    \label{tab:supp:dft4d}
\end{table}

\subsubsection{Baselines}

\paragraph{MLP-ReLU and MLP-ReLU PE.6.}
MLPs with ReLU~\cite{krizhevsky2012imagenet} are frequently used to warp points in existing works. 
In our experiment, the architecture contains 6 hidden layers of 128 hidden dimensions.
In Tab.~\ref{tab:df4d}, the Fourier positional encoding~\cite{mildenhall2020nerf} is not used in `MLP-ReLU', but is applied in `MLP-ReLU PE.6' with 6-level resolutions.

\paragraph{DCT-NeRF~\cite{wang2021neural}.}
The Fourier positional encoding is not applied.
Rather than outputting the target point location ${\bm y}$, this baseline method produces the coefficients of a DCT basis that is jointly learned from the data.
Similar technology is also employed in~\cite{li2023dynibar}.

\paragraph{BANMO~\cite{yang2022banmo}.}
BANMO is a solution to reconstruct the avatar of a generic object, \eg cat, from a monocular video. 
The avatar bones are modelled by a set of 3D Gaussians, and the skinning weight is a combination of a Gaussian-based weighting function and a neural network. The 3D location of a query point is encoded by Fourier encoding~\cite{mildenhall2020nerf}. The rest pose code is derived by a linear layer, and the pose code is derived by the Fourier encoding of the frame and a linear layer.
In our experiment, we adopt its avatar deformation module into our setting and use the hyper-parameters as in the original paper~\cite{yang2022banmo}.
Provided a set of training point trajectories, we optimize the 3D Gaussians and the relevant networks as in~\cite{yang2022banmo}. 
During testing, we animate the testing points in the canonical frame to produce the trajectories, based on the learned Gaussians and networks.

\paragraph{BoneCloud.}
Based on BANMO~\cite{yang2022banmo} and KeyTr~\cite{novotny2022keytr}, we propose this BoneCloud method, which is a learnable bone basis.
Compared to BANMO, this BoneCloud method does not employ any nonlinear neural network.
Instead, it has a point cloud in the canonical frame, and each point stores a time sequence of SE(3) transformations.
The skinning weights are created by a pre-fixed radial basis function. Consequently, a 3D point ${\bm x}$ in the canonical frame can be transformed to ${\bm y}$ at frame $t$, via linear blend skinning. Specifically, it is given by
\begin{equation}
    \begin{pmatrix}
     {\bm y} \\
     1
    \end{pmatrix}
    = \left(\sum_{k} w_{k} {\bm T}_{k}^t \right) 
    \begin{pmatrix}
     {\bm x} \\
     1
    \end{pmatrix}
\end{equation}
and
\begin{align}
    \Tilde{w}_k &= \exp{ (-\sigma \|{\bm x} - {\bm v}_k\|_2)  } \\
    {w}_k &= \frac{\Tilde{w}_k}{\sum_k \Tilde{w}_k},
\end{align}
in which ${\bm v}\in\mathbb{R}^3$ is a bone in the bone cloud, $k$ is the index of the bone, ${\bm T}_k^t \in SE(3)$ denotes the transformation of the bone $k$ at time $t$.
In our experiment, we leverage 1024 points as bones.
During training, we leverage the provided point trajectories to optimize the bone locations at the canonical frame and the bone transformations at individual time steps.
During testing, we transform the testing points in the canonical frame to individual target frames, so as to produce the point trajectories.

\subsubsection{More Results on the Synthetic Dataset}
\label{sec:supp:more_on_synthetic}

Corresponding to Fig.~\ref{fig:synthetic} in the main paper, we show the qualitative results of all DOMA variants in Fig.~\ref{fig:supp:synthetic}. We can see that the affinity field is able to represent all explored linear transformations. 
This indicates the output layer highly influences the motion types that the model can represent.

In order to investigate how the hidden dimension influences the DOF representation, we increase the hidden dimension of DOMA-Trans from 128 to 256. The results are shown in Tab.~\ref{tab:supp:synthetic2}. 
Without smoothness regularization, we can see that a higher hidden dimension slightly improves the performance in some cases, but degrades the performance on translation, probably due to overfitting. 
When applying the smoothness regularization to overcome overfitting, motion prediction on translation is significantly improved, whereas the performances on other linear transformations are much worse. 
On the other hand, the performances of DOMA-Affinity on all motion types are consistently and considerably improved by the smoothness regularization.
Based on these observations, we can conclude that
\begin{itemize}
    \item Both the hidden dimension and the DOFs represented by ${\bm A}$ can influence the model representation power.
    \item Increasing the hidden dimension improves the performance but not always. Overfitting could occur.
    \item The smoothness regularization can improve the performance significantly if the ground truth DOF is explicitly modeled at the output layer. Otherwise, it can degrade the performance. 
    \item Increasing the hidden dimension cannot simply increase the DOF representations. Otherwise, the smoothness regularization should lead to consistent improvements for all linear transformations.
\end{itemize}

\begin{table}
    \centering
    \footnotesize
    \begin{tabular}{lcccc}
    \toprule
         \textit{Methods} & Rotation & Scaling & Shearing & Translation  \\
         \midrule
         -Trans & 2725.4 & 1817.8 & 1619.5 & 1042.4 \\
         -SE(3) & {730.6} & 1991.4 & 1138.3 & 899.4 \\
         -Scaled SE(3) & {801.1} & {685.8} & 1524.7 & {1096.2} \\
         -Affinity & 1486.0 & {915.4} & {622.1} & 822.4\\
         \midrule
         -Trans-E & 38.0 & 1669.6 & {753.6} & 38.8 \\
         -SE(3)-E & 20.0 & 1761.3 & 832.7 & 26.4 \\
         -scaled SE(3)-E & 21.2 & {1161.8} & 961.1 & 24.0 \\
         -Affinity-E & {19.2} & 155.7 & 864.0 & {15.7} \\
         \midrule
         -Trans-H & 4919.9 & 2056.4 & 2446.8 & 37.8 \\
         -SE(3)-H & 52.4 & 2012.4 & 1665.0 & 36.9 \\
         -scaled SE(3)-H & 29.3 & 22.1 & 688.0 & 30.3 \\
         -Affinity-H & \textbf{5.4} & \textbf{26.3} & \textbf{8.5} & \textbf{28.8} \\
        \bottomrule
    \end{tabular}
    \caption{Results on Synthetic sequences w.r.t. EPE (in $\times 10^{-4}$). `-E' denotes the elasticity loss proposed in Nerfies~\cite{park2021nerfies}, and `-H' denotes our smoothness loss. Best results are in boldface. }
    \label{tab:synthetic}
\end{table}

\begin{table}
    \centering
    \footnotesize
    \scalebox{0.95}{\begin{tabular}{lcccc}
    \toprule
         \textit{Methods} & Rotation & Scaling & Shearing & Translation  \\
         \midrule
         -Trans-128d & 2725.4 & 1817.8 & 1619.5 & 1042.4\\
         -Trans-256d & 2187.1 & 1846.9 & 1515.8 & 1231.8\\
         -Affinity-128d & 1486.0 & {915.4} & {622.1} & 822.4\\
         \midrule
         -Trans-H-128d(0.1) & 4919.9 & 2056.4 & 2446.8 & 37.8\\
         -Trans-H-256d(0.1) & 4911.8 & 2078.7 & 1733.7 & 217.4 \\
         -Trans-H-256d(1) & 10945.3 & 8400.4 & 8701.2 & \textbf{16.2} \\
         -Affinity-H-128d(0.1) & \textbf{5.4} & \textbf{26.3} & \textbf{8.5} & {28.8} \\
    \bottomrule
    \end{tabular}}
    \caption{Results on Synthetic sequences as in Tab.~\ref{tab:synthetic} in the main paper. Numbers denote EPE in $\times 10^{-4}$. `-128d' and `-256d' denote the hidden dimension of the SIREN network. The number in $()$ denotes the weight of the smoothness loss term. Best results are in {boldface}.}
    \label{tab:supp:synthetic2}
\end{table}

\begin{figure*}
    \centering
    \includegraphics[width=\linewidth]{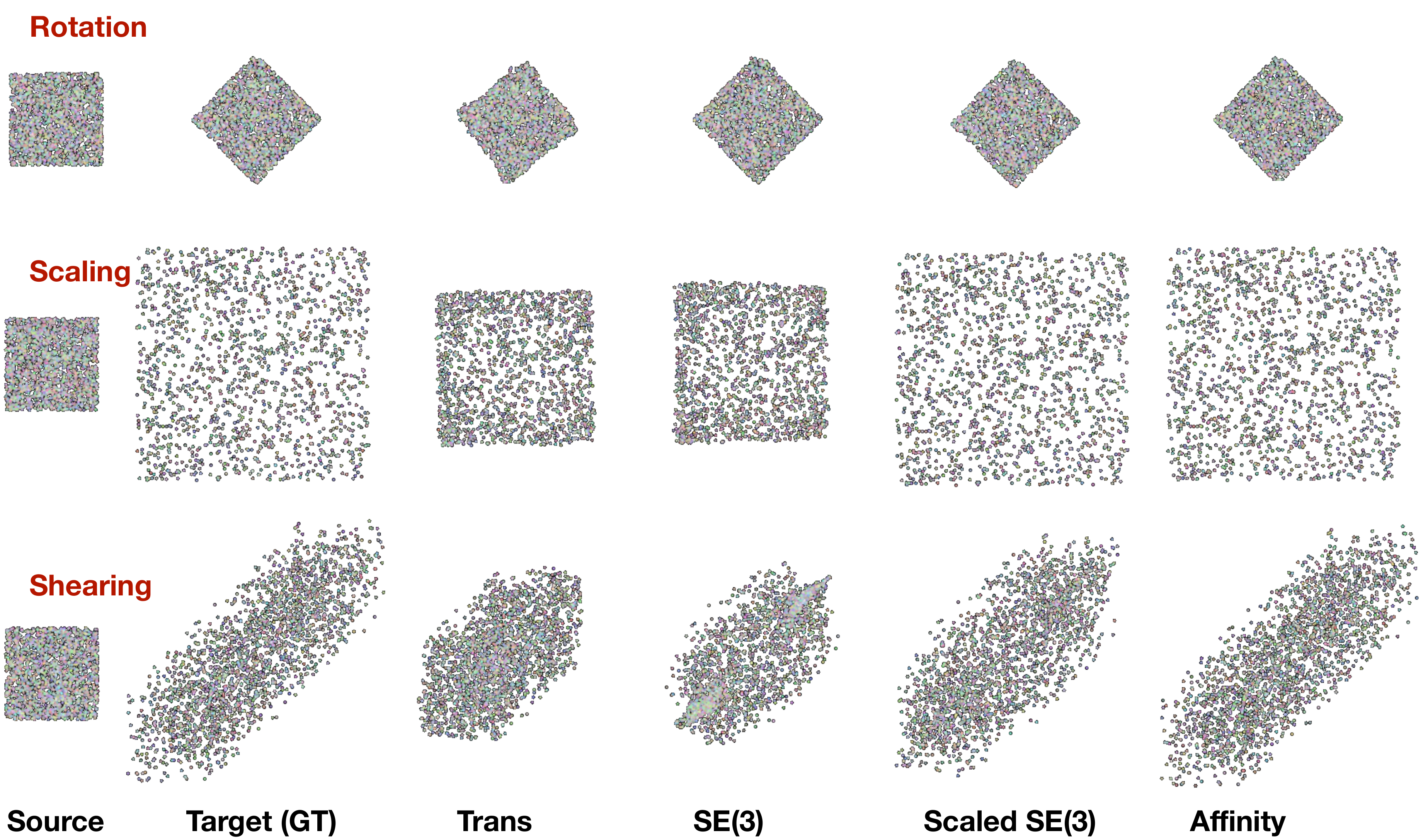}
    \caption{Illustrations of results on the Synthetic sequences. The smoothness regularization is applied. Rows show types of motions, and columns show the testing points in the canonical frame, a target frame, and estimated results from different methods, respectively. }
    \label{fig:supp:synthetic}
\end{figure*}

\paragraph{Runtime analysis.}
In addition, we compare our derived analytical gradients with auto-diff of Pytorch~\cite{paszke2019pytorch} w.r.t. the runtime. We set the smoothness loss weight to 0.1, and train DOMA-Affinity for 1000 iterations. 
This experiment is conducted with Ubuntu 20.04, NVIDIA TITAN RTX 24GB, CUDA 11.4, 32GB RAM.
The results are shown in Tab.~\ref{tab:supp:runtime}. 
We can see the analytical gradients improve the efficiency consistently. 
Compared to the standard auto-diff, the runtime is reduced by 28\%.

\begin{table}
    \centering
    \scriptsize
    \begin{tabular}{lccccc}
    \toprule
         \textit{Methods} & Rotation & Scaling & Shearing & Translation & average \\
         \midrule
         auto-diff~\cite{paszke2019pytorch} & 133.68 & 132.36 & 133.44 & 133.51 & 133.25\\
         analytical grad & 96.08 & 96.26 & 95.78 & 95.92 & 96.01\\
         
    \bottomrule
    \end{tabular}
    \caption{Comparison between our derived analytical gradients and auto-diff of Pytorch. Runtime is measured in seconds.}
    \label{tab:supp:runtime}
\end{table}

\subsection{Additional Presentations on Guided Mesh Alignment (Sec.~\ref{sec:exp2})}

\subsubsection{Dataset}
We employ the ReSynth dataset~\cite{ma2021scale,ma2021power} in this study. 
Specifically, we choose 16 sequences from 4 subjects in the \textit{packed} sequences in the \textit{test} split (see Tab.~\ref{tab:supp:resynth}).
For each sequence, we first perform down-sampling by every 2 frames, and then select the first 30 frames for experiments. The first frame in each sequence is regarded as the canonical frame.

The motion complexity depends on the subject and the clothing type. 
As shown in Fig.~\ref{fig:supp:resynth}, sequences with `rp\_felice\_posed\_004' are more complex than others, because of the loose long skirt. In this case, the points on the long skirt are far away from the body surface, which is aligned and guided by the SMPL-X~\cite{SMPL-X:2019} mesh vertices.
Other subjects have tight clothes.

\begin{figure*}
    \centering
    \includegraphics[width=\linewidth]{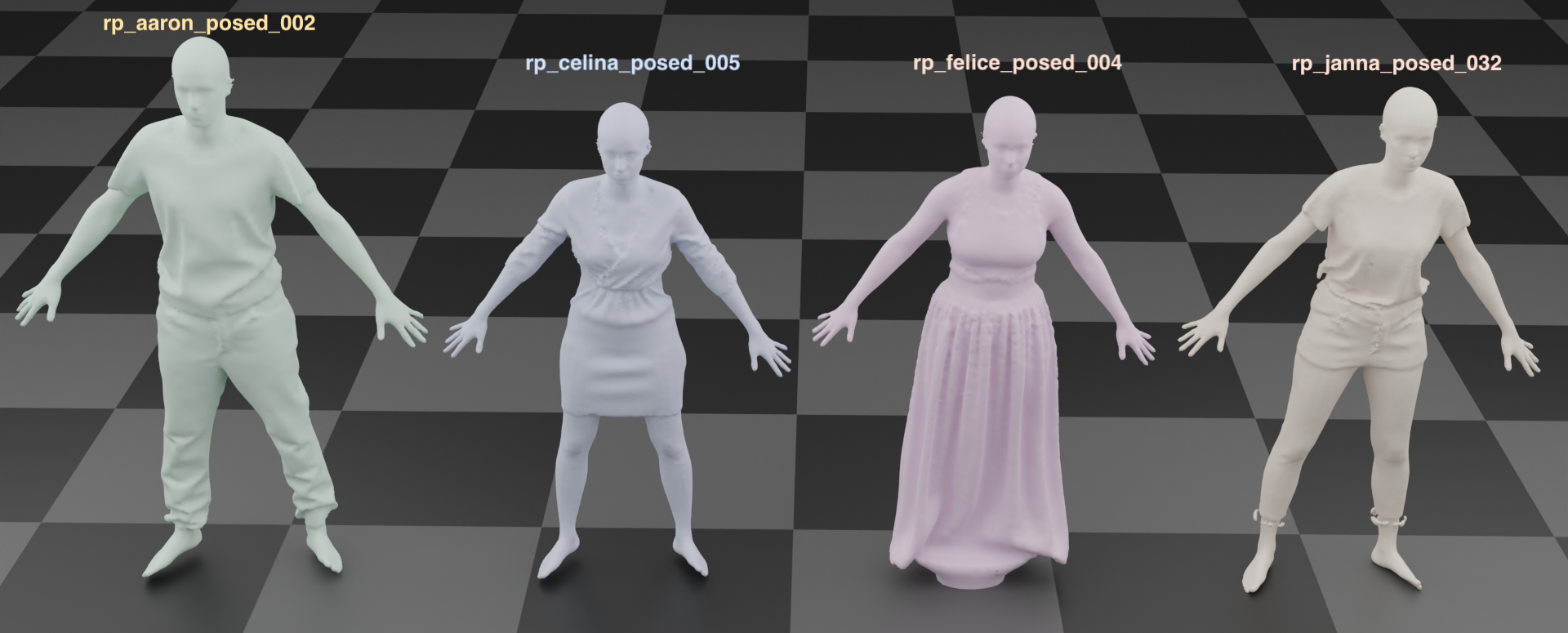}
    \caption{Illustrations of the 4 subjects in our employed sequences. These meshes are created by Poisson surface reconstruction based on the provided oriented points in the canonical frames. We have created 16 such meshes for individual sequences, which are roughly at the same pose, \ie the A-pose.}
    \label{fig:supp:resynth}
\end{figure*}

\begin{table}
    \centering
    \begin{tabular}{lc}
    \toprule
          Subjects & Actions \\
           \midrule
         rp\_aaron\_posed\_002 & \begin{tabular}{@{}c@{}} 96\_jerseyshort\_hips \\  96\_jerseyshort\_squats \\ 96\_longshort\_flying\_eagle \\ 96\_longshort\_tilt\_twist\_left \end{tabular} \\
           \midrule
         rp\_celina\_posed\_005 & \begin{tabular}{@{}c@{}} 96\_jerseyshort\_hips \\  96\_jerseyshort\_squats \\ 96\_longshort\_flying\_eagle \\ 96\_longshort\_tilt\_twist\_left \end{tabular} \\
           \midrule
         rp\_felice\_posed\_004 & \begin{tabular}{@{}c@{}} 96\_jerseyshort\_hips \\  96\_jerseyshort\_squats \\ 96\_longshort\_flying\_eagle \\ 96\_longshort\_tilt\_twist\_left \end{tabular} \\
           \midrule
         rp\_janna\_posed\_032 & \begin{tabular}{@{}c@{}} 96\_jerseyshort\_hips \\  96\_jerseyshort\_squats \\ 96\_longshort\_flying\_eagle \\ 96\_longshort\_tilt\_twist\_left \end{tabular} \\    
        \bottomrule
    \end{tabular}
    \caption{Employed Resynth sequences in Sec.~\ref{sec:exp2}.}
    \label{tab:supp:resynth}
\end{table}

\subsubsection{Performances of All Model Variants}
In Tab.~\ref{tab:resynth}, we only show the performance of the DPF baseline, DOMA-Trans, and DOMA-Affinity. Here we show the performances of all DOMA variants under the same experiment setting.
The results are presented in Tab.~\ref{tab:supp:resynth_all}.
We can draw similar conclusions as in Sec.~\ref{sec:exp2}. The performance of the affinity field is similarly better than other variants, in particular on the Chamfer distances.

\begin{table}
    \centering
    \footnotesize
    \begin{tabular}{lcccc}
    \toprule
          & $\mathcal{L}_{CD}\downarrow$ & $\mathcal{L}_n\downarrow$ & \textit{STD(E)}$\downarrow$  & \textit{STD(V)}$\downarrow$   \\
           \midrule
         DPF~\cite{prokudin2023dynamic} & 1.149 & \textbf{0.122} & \textbf{11.6} & 24.6\\
         -Trans & 1.230 & 0.128 & 12.8 & 22.9  \\
         -SE(3) & 1.343 & 0.134 & 16.2 & 22.9  \\
         -Scaled SE(3) & 1.273 & 0.127 & 16.2 & 22.8  \\
         -Affinity & \textbf{1.142} & 0.125 & 11.9 & \textbf{22.8} \\
         \midrule
         DPF-A~\cite{prokudin2023dynamic} & 1.166 & \textbf{0.119} & \textbf{10.3} & 24.2\\
         -Trans-A & 1.195 & 0.123 & 10.4 & \textbf{23.0}  \\
         -SE(3)-A & 1.278 & 0.123 & 11.5 & \textbf{23.0}  \\
         -Scaled SE(3)-A & 1.20 & 0.120 & 11.3 & 23.0  \\
         -Affinity-A & \textbf{1.151} & 0.122 & 10.6 & \textbf{23.0} \\
         \midrule
          DPF-H~\cite{prokudin2023dynamic} & 1.142 & \textbf{0.123} & 10.3 & 24.2\\
         -Trans-H & 1.207 & 0.128 & 10.8 & \textbf{22.9}  \\
         -SE(3)-H & 1.230 & 0.127 & 12.2 & 23.0  \\
         -Scaled SE(3)-H & 1.189 & 0.125 & 11.5 & 22.9  \\
         -Affinity-H & \textbf{1.127} & 0.127 & \textbf{10.1} & \textbf{22.9} \\
         \midrule
         DPF-AH~\cite{prokudin2023dynamic} & 1.189 & \textbf{0.120} & 9.3 & 24.3\\
         -Trans-AH & 1.240 & 0.124 & 9.3 & \textbf{23.0}  \\
         -SE(3)-AH & 1.265 & 0.124 & 10.9 & 23.1  \\
         -Scaled SE(3)-AH & 1.255 & 0.126 & 8.7 & 23.0  \\
         -Affinity-AH & \textbf{1.187} & 0.124 & \textbf{8.9} & \textbf{23.0} \\
        \bottomrule
    \end{tabular}
    \caption{Results of guided mesh alignment on our selected Resynth sequences. $\mathcal{L}_{CD}$ is in $\times 10^{-4}$. \textit{STD(E)} and \textit{STD(V)} are given in millimeters. This table is supplementary to Tab.~\ref{tab:resynth} in the main paper.}
    \label{tab:supp:resynth_all}
\end{table}

\subsubsection{Analysis on Clothing Types}
In addition to the averaged performance on all sequences, we have also observed consistent trends on individual subjects that have different clothing types.
In this experiment, we leave `rp\_felice\_posed\_004' out of others, and perform evaluations separately.
For compactness, we only show the comparison between our proposed affinity field and the frame-wise DPF models~\cite{prokudin2023dynamic}, with the weights of the AIAP loss term and the motion smoothness term being $(1, 0.001)$.
The results are shown in Tab.~\ref{tab:supp:persubject}.
We can see that the affinity field outperforms DPF~\cite{prokudin2023dynamic} on the subject with the long skirt, whereas performs worse on subjects with tight clothing. 
A probable reason is that points that are close to the body surface can be effectively guided by the SMPL-X mesh vertices. Due to much more model parameters, the frame-wise DPF models can overfit to the guidance points, and hence produces better results on the body surfaces and the tight clothes. 
Simultaneously, it produces more artifacts and discontinuities at regions that are far away from the guidance points, leading to inferior performance to the affinity field.

\begin{table}
    \centering
    \footnotesize
    \setlength{\tabcolsep}{2.5pt}
    \begin{tabular}{lccccc}
    \toprule
      Subjects & Methods & \textit{CD}$\downarrow$ & \textit{CDN}$\downarrow$ & \textit{STD(E)}$\downarrow$ & \textit{STD(V)}$\downarrow$ \\
      \midrule
      \multirow{2}{*}{rp\_felice\_posed\_004} & DPF-AH~\cite{prokudin2023dynamic} & 3.086 & 0.194 & \textbf{14.4} & 25.8 \\
      & -Affinity-AH & \textbf{2.857} & \textbf{0.190} & 15.3 & \textbf{21.9} \\
      \midrule
      \multirow{2}{*}{others} & DPF-AH~\cite{prokudin2023dynamic} & \textbf{0.557} & \textbf{0.096} & 7.6 & 23.8 \\
      & -Affinity-AH & 0.630 & 0.102 & \textbf{6.8} & \textbf{23.3} \\
    \bottomrule
    \end{tabular}
    \caption{Evaluation of methods on the Resynth sequence `rp\_felice\_posed\_004' as discussed in Section~\ref{sec:exp2}.}
    \label{tab:supp:persubject}
\end{table}

\begin{table}
    \centering
    \footnotesize
    \setlength{\tabcolsep}{2.2pt}
    \scalebox{0.95}{
    \begin{tabular}{lccccc}
    \toprule
      Subjects & Methods & \textit{CD}$\downarrow$ & \textit{CDN}$\downarrow$ & \textit{STD(E)}$\downarrow$ & \textit{STD(V)}$\downarrow$ \\
      \midrule
      \multirow{5}{*}{all sequences} & DPF-AH~\cite{prokudin2023dynamic} & 1.047 & \textbf{0.100} & 10.1 & 24.2 \\
      & -Trans-AH & 1.058 & 0.106 & 10.6 & \textbf{23.2} \\
      & -SE(3)-AH & 1.055 & 0.105 & 10.9 & \textbf{23.2} \\
      & -Scaled SE(3)-AH & 1.060 & 0.104 & 10.3 & \textbf{23.2} \\
      & -Affinity-AH & \textbf{1.023} & 0.105 & \textbf{9.8} & \textbf{23.2} \\
      
      \midrule
      \multirow{5}{*}{rp\_felice\_posed\_004} & DPF-AH~\cite{prokudin2023dynamic} & 2.740 & 0.151 & \textbf{{15.6}} & 25.3 \\
      & -Trans-AH & 2.724 & 0.157 & 16.0 & 22.6 \\
      & -SE(3)-AH & 2.683 & 0.149 & 18.5 & \textbf{22.5} \\
      & -Scaled SE(3)-AH & {2.694} & \textbf{0.148} & 17.5 & 22.6 \\
      & -Affinity-AH & \textbf{2.544}  & 0.152 & 16.5 & \textbf{22.5} \\
      \midrule
      \multirow{5}{*}{others} & DPF-AH~\cite{prokudin2023dynamic} & \textbf{0.483} & \textbf{0.083} & 8.3 & 23.8 \\
      & -Trans-AH & 0.503 & 0.089 & 8.8 & \textbf{23.4} \\
      & -SE(3)-AH & 0.513 & 0.090 & 8.4 & 23.5 \\
      & -Scaled SE(3)-AH & 0.516 & 0.089 & 7.9 & 23.5 \\
      & -Affinity-AH & 0.515 & 0.090 & \textbf{7.6} & 23.5 \\      
    \bottomrule
    \end{tabular}}
    \caption{Evaluations based on the models with 256D hidden variables. Other settings are identical with Tab.~\ref{tab:resynth} and~\ref{tab:supp:persubject}.
    Best results are highlighted in boldface.}
    \label{tab:supp:256dim}
\end{table}

\subsubsection{Influence of Hidden Dimensions}
We set the hidden dimension to 128 by default in the main paper and the above experiments. Here we increase it to 256 and re-evaluate the performances. According to our analysis of the motion model bound (see Sec.~\ref{sec:method} and~\ref{sec:supp:theorem}), increasing the hidden dimension is able to improve the model representation power on the motion complexity.

The results are presented in Tab.~\ref{tab:supp:256dim}.
Compared to models with 128D hidden variables (see Tab.~\ref{tab:resynth} and Tab.~\ref{tab:supp:persubject}), models with 256D hidden variables consistently produce better results. With this new setting, the performance gaps between individual methods tend to vanish. The temporal smoothness tends to degrade though.

In the meanwhile, we can see that the affinity field still has comparably better performance than frame-wise DPF~\cite{prokudin2023dynamic}, but produces smoother results, leading to the same observation and conclusion as demonstrated in Sec.~\ref{sec:exp2}.
Focusing on the performances on different sequences, we can see DPF~\cite{prokudin2023dynamic} still outperforms DOMA models on tight clothing w.r.t. alignment, but the gap becomes smaller compared to Tab.~\ref{tab:resynth}.
The affinity field outperforms DPF on the loose long skirt sequence by a large margin.
Furthermore, from the model size perspective, our DOMA models are still significantly more lightweight than the DPF~\cite{prokudin2023dynamic} baseline.

\section{Additional Experiments}

\begin{figure*}
    \centering
    \includegraphics[width=\linewidth]{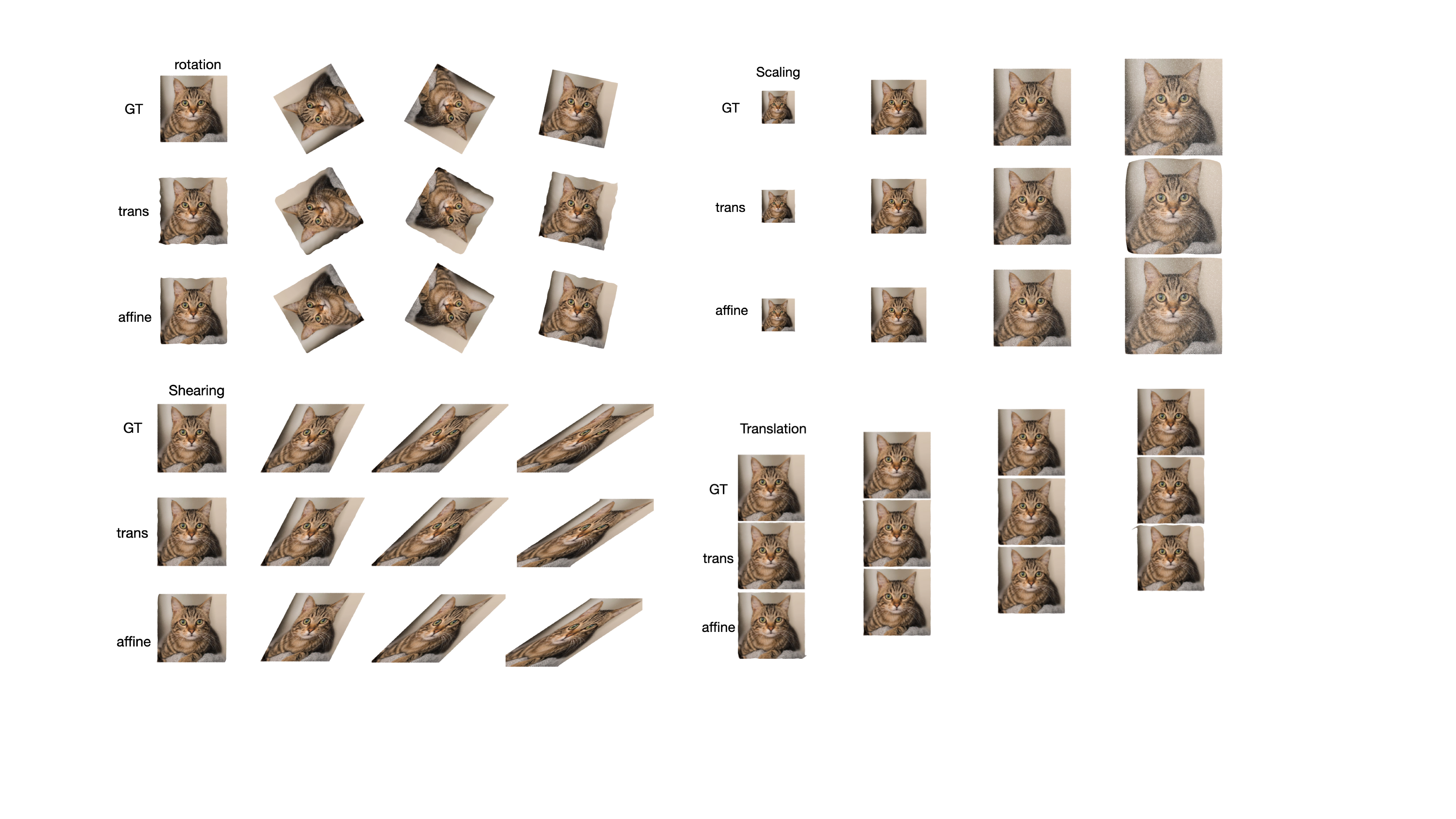}
    \caption{Illustration of modeling 2D image deformation, in which the hidden dimension is 64. `GT', `trans', and `affine' denote the ground truth, DOMA-Trans, and DOMA-Affinity, respectively.}
    \label{fig:supp:cat}
\end{figure*}

\begin{table*}
    \centering
    \footnotesize
    \begin{tabular}{lcccccccc}
    \toprule
    & \multicolumn{2}{c}{Rotation} & \multicolumn{2}{c}{Scaling} & \multicolumn{2}{c}{Shearing} & \multicolumn{2}{c}{Translation} \\
    & -Trans & -Affinity & -Trans & -Affinity & -Trans & -Affinity & -Trans & -Affinity \\
    \midrule
    hdim=32 & 60.8 & 85.0 & 31.2 & 8.9 & 11.7 & 5.4 & 100.4 & 44.0 \\
    hdim=64 & 64.5 & 33.9 & 15.3 & 6.2 & 10.5 & 4.3 & 25.8 & 18.7 \\
    hdim=128 & 35.4 & 20.3 & 8.2 & 4.6 & 6.4 & 3.8 & 20.2 & 20.4 \\
    \bottomrule
    \end{tabular}
    \caption{Evaluations in the 2D image deformations. As in Tab.~\ref{tab:synthetic}, the numbers denote EPE in $\times 10^{-4}$, and are the lower the better. }
    \label{tab:supp:image}
\end{table*}

\subsection{Learning 2D Image Deformation}

Similar to the experiments on the 3D synthetic dataset, here we conduct an experiment on 2D image deformation, in order to further investigate the representation power of DOMA models.

\paragraph{Data, evaluation, and methods.}
We use a RGB image of cat that has 512$\times$512 of pixels.
As our synthetic dataset, we perform translation, rotation, scaling, and shearing in the 2D domain, and produce 30 frames.
In each case, we randomly choose 25\% points for training the motion field, and use the remaining for testing.
The evaluation metric is the same as in Sec.~\ref{sec:exp1}.
DOMA-Trans and DOMA-Affinity with different hidden dimensions are applied in this experiment. 
No regularization is used during training.

\paragraph{Results.}
The quantitative evaluation is shown in Tab.~\ref{tab:supp:image}, and some qualitative results are shown in Fig.~\ref{fig:supp:cat}.
We can see that the affinity model outperforms the translation model consistently with different hidden dimensions. 
In particular, the performance of the affinity model is superior when the hidden dimension is smaller.
These results demonstrate the advantages of additional DOFs.

\begin{figure*}
    \centering
    \includegraphics[width=\linewidth]{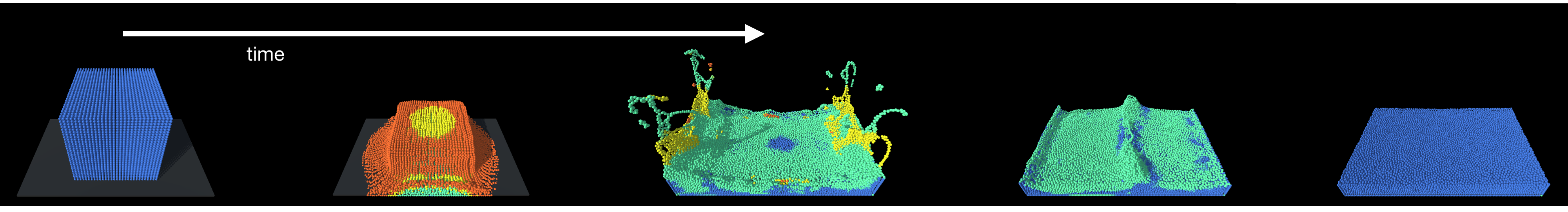}
    \caption{The particle system to simulate a fluid field in Unity3D, which is implemented based on~\cite{fluidsim}.  }
    \label{fig:supp:particle}
\end{figure*}

\subsection{Inferring Dynamics of Fluid Fields}
\label{sec:supp:fluid}
In Sec.~\ref{sec:exp1}, we have investigated the model representation power based on the DeformingThings4D~\cite{li20214dcomplete} sequences. Despite various model shapes and movements, they are limited to elastic deformations of solid objects.
In this section, we propose a more challenging scenario, modeling a fluid field. 
To perform empirical studies, we follow~\cite{fluidsim} to simulate how liquid moves in a bounded field with Unity3D, and record the particle trajectories (see Fig.~\ref{fig:supp:particle}).

The entire sequence contains 931 frames and 27,000 particles. We randomly choose 50\% for training the motion field and use the rests for testing.
We find that all methods investigated in this paper are not able to reconstruct the entire sequence. Thus, we down-sample the entire sequence by every 2 frames, and then trim the down-sampled sequence into 10-frame clips. 
Specifically, the frame indices of the clips are $\{(t,t+10)\}_{t=10,15,\ldots,325}$, in which the sequences with trivial motions, \eg static state in the beginning and steady state in the end, are excluded. This pre-processing will lead to 22 clips in total.

In each clip, the first frame is regarded as the canonical frame. Points in the canonical frame are transformed into individual target frames, and their averaged L1 distances to the ground truth are minimized during training. 
The evaluation metric is the scene flow end point error (EPE), identical to Sec.~\ref{sec:exp1}.
We use the Adam optimizer~\cite{kingma2014adam} for training. The initial learning rate is $1e-4$. Training terminates after 5000 epochs.

In this experiment, we compare frame-wise DPF~\cite{prokudin2023dynamic}, DOMA-Trans, and DOMA-Affinity. We set these models to have 2 hidden layers, and set their hidden dimensions to be 64D or 128D. 
Results are presented in Tab.~\ref{tab:supp:particle}.
We can see that these two spatiotemporal motion field models considerably outperform the DPF baseline.
In addition, the affinity model performs comparably better than the translation field model.
With fewer hidden variables, the advantage of the additional DOFs are more obvious.
Fig.~\ref{fig:supp:particle_vis} illustrates some examples of how these methods perform.
We can see that the frame-wise DPF method can lead to significant discontinuities between frames, and less accurate motion prediction than DOMA.

\begin{table}
    \centering
    \begin{tabular}{lcc}
    \toprule
          & 64D & 128D  \\
         \midrule
         DPF~\cite{prokudin2023dynamic} & 400.20 & 501.60\\
         \midrule
         DOMA-Trans & 176.43 & 188.22\\
         DOMA-Affinity & {168.89} & 183.33\\
        \bottomrule
    \end{tabular}
    \caption{Motion prediction of unseen points on fluid simulation sequences. The numbers are EPEs in $\times 10^{-4}$. Results of 64D and 128D hidden dimensions are both presented.}
    \label{tab:supp:particle}
\end{table}

\begin{figure*}
    \centering
    \includegraphics[width=\linewidth]{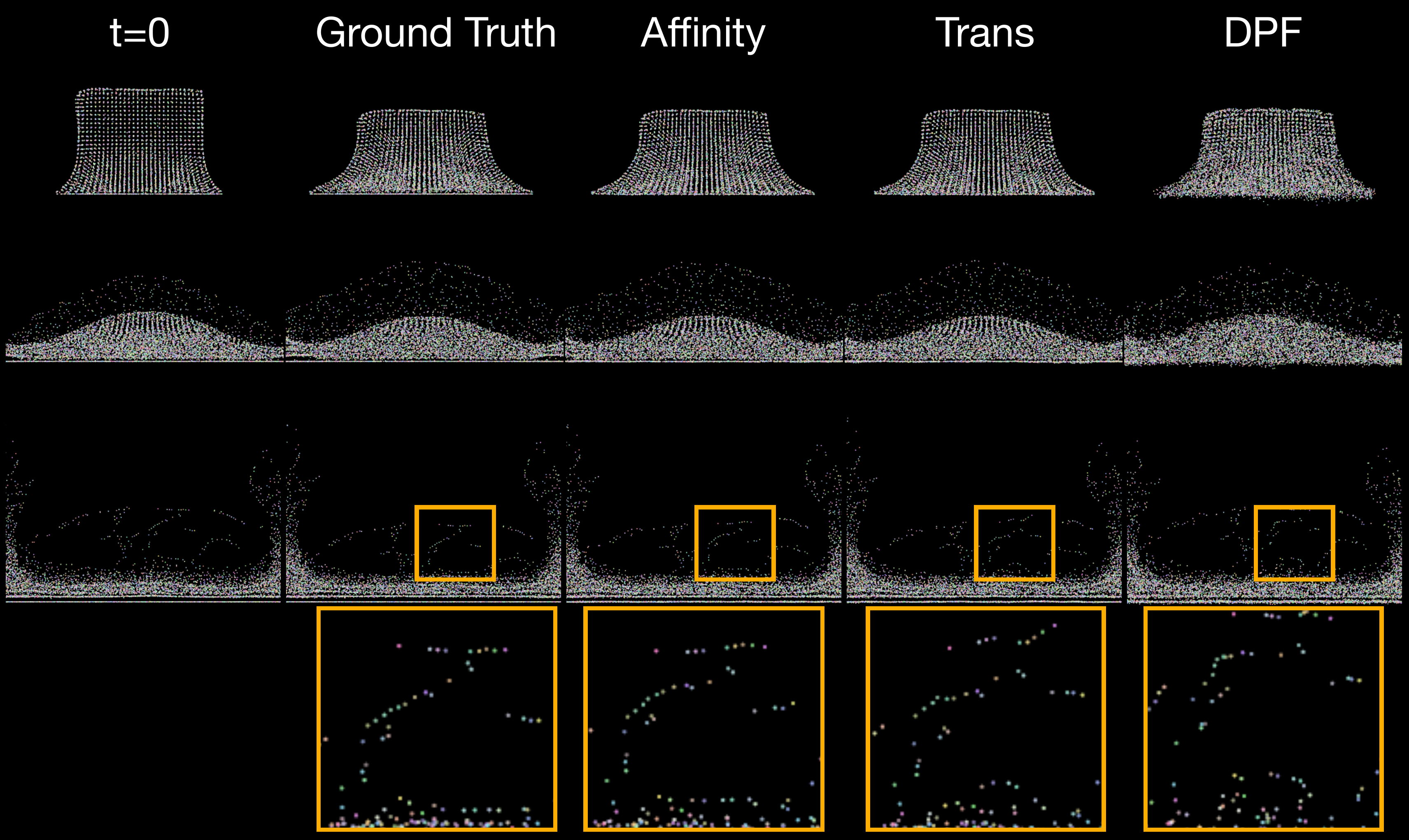}
    \caption{Illustrations of three fluid dynamics estimation results. The hidden variables have 64 dimensions. The first three rows denote three sequences. Columns from left to right denote the first frame, the last frame of ground truth, DOMA-Affinity, DOMA-Trans, and frame-wise DPF, respectively. The zoomed-in regions in the third row are shown below their respective images.  }
    \label{fig:supp:particle_vis}
\end{figure*}

\end{document}